\documentclass[onefignum,onetabnum]{siamart220329}

\usepackage{fullpage}
\usepackage{epsfig}
\usepackage{graphicx}
\usepackage{amsmath}
\usepackage{amssymb}
\usepackage{color}
\usepackage{pdflscape}
\usepackage{algorithm,algpseudocode}%
\usepackage{booktabs}
\usepackage{comment}
\usepackage{threeparttable}
\usepackage{cancel}
\usepackage{hyperref}
\usepackage{bbm}
\usepackage[normalem]{ulem}

\usepackage{amssymb}
\usepackage{pifont}%

\usepackage[colorinlistoftodos,bordercolor=orange,backgroundcolor=orange!20,linecolor=orange,textsize=scriptsize]{todonotes}

\newtheorem{thm}{Theorem}

\newtheorem{lem}[thm]{Lemma}

\newtheorem{cor}[thm]{Corollary}

\newtheorem{rem}{Remark}
\newtheorem{assumption}{Assumption}

\newcommand{\inner}[1]{\left\langle #1\right\rangle}

\newcommand{\R}{\mathbb{R}}

\newcommand{\ab}{\mathbf{a}}
\newcommand{\phib}{\boldsymbol{\phi}}

\newcommand{\xb}{\mathbf{x}}
\newcommand{\zb}{\mathbf{z}}

\newcommand{\ub}{\mathbf{u}}
\newcommand{\vb}{\mathbf{v}}

\newcommand{\PP}{\mathbb{P}}

\newcommand{\E}{\mathbb{E}}

\algnewcommand{\Inputs}{%
	\State \textbf{Inputs:}
}
\algnewcommand{\Initialize}{%
	\State \textbf{Initialize:}
}
\algnewcommand{\Outputs}{%
	\State \textbf{Outputs:}
}
\algnewcommand{\ForLoop}{%
	\textbf{For $j=1,2,...,m$ do}
}

\algnewcommand{\ForOuterLoop}{%
	\textbf{For $k=1,2,..., K$ do}
}

\algnewcommand{\ForLoopKZPT}{%
	\textbf{For $j=1,2,..., p \lfloor {m \over p}\rfloor $ do}
}

\algnewcommand{\ForLoopmod}{%
	\textbf{For $L=0,1,...,L-1$ do}
}
\algnewcommand{\ForEnd}{%
	\textbf{End For}
}
\algnewcommand{\Iterate}{%
	\State \textbf{Iterate:}
}

\allowdisplaybreaks

\begin{document}

\title{Stochastic gradient descent for streaming linear and rectified linear systems with adversarial corruptions}

\author{
Halyun Jeong\thanks{University of California Los Angeles, Department of Mathematics, Los Angeles, CA 90095, and The State University of New York at Albany,  Department of Mathematics \& Statistics, Albany, NY 12222, (\email{hjeong2@albany.edu}) } \and Deanna Needell\thanks{University of California Los Angeles, Department of Mathematics, Los Angeles, CA 90095
(\email{deanna@math.ucla.edu})}
 \and Elizaveta Rebrova
\thanks{Princeton University, Operations research (ORFE) Department,Sherrerd Hall, Charlton Street, Princeton, NJ 08544
(\email{elre@princeton.edu})}}
\maketitle

\begin{abstract}
We propose SGD-exp, a stochastic gradient descent approach for linear and ReLU regressions under Massart noise (adversarial semi-random corruption model) for the fully streaming setting. We show novel nearly linear convergence guarantees of SGD-exp to the true parameter with up to $50\%$ Massart corruption rate, and with any corruption rate in the case of symmetric oblivious corruptions. This is the first convergence guarantee result for robust ReLU regression in the streaming setting, and it shows the improved convergence rate over previous robust methods for $L_1$ linear regression due to a choice of an exponentially decaying step size, known for its efficiency in practice. Our analysis is based on the drift analysis of a discrete stochastic process, which could also be interesting on its own.
\end{abstract}


\section{Introduction}
Robust regression aims to develop regression methods that provide a proper fit of the data, even in the presence of outliers.  Such outliers can arise, for example, from labeling mistakes during the data-gathering process \cite{brodley1999identifying}, incorrect measurements in tomography \cite{lin2015quantifying}, transmission errors \cite{gogoi2011survey}, or adversarial attacks \cite{biggio2012poisoning, steinhardt2017certified} in distributed machine learning. On the other hand, the high-dimensionality and just sheer volume of modern datasets pose significant challenges in machine learning. Stochastic algorithms such as stochastic gradient descent (SGD) are standard approaches to address this high dimensionality, as they utilize only a part of the dataset at a time \cite{gurbuzbalaban2021heavy, tsai2022heavy, diakonikolas2022streaming}. Moreover, in many modern applications, data is streamed, meaning that methods cannot retain past data and are only allowed to work with a given small portion of data at a time in such scenarios. In the fully streaming setting, an algorithm typically processes one data point at a time and cannot revisit past data points \cite{haselby2023fast, pesme2020online}. Such constraints arise from limited memory, the immediacy of real-time processing, or just the vast volume of the data.
These challenges have naturally led to the development of robust regression methods suited for the streaming setting.

\subsection{Robust linear and ReLU regression}

Robust regression has a long history and has played an important role in both statistics and machine learning. Its goal is to learn the fitting parameters, even when the observations are contaminated by a constant fraction of adversarial outliers \cite{tukey1960survey,huber1992robust, chen2013robust, klivans2018efficient, diakonikolas2019efficient,  liu2020high}. In particular, fitting generalized linear models is a fundamental subject in statistics \cite{kalai2009isotron, kakade2011efficient}. One important case arises when the nonlinear function is the ReLU activation function. Recently, ReLU regression has garnered a lot of attention due to its relevance to neural networks \cite{soltanolkotabi2017learning, yehudai2020learning, karmakar2022provable, diakonikolas2022learning}. 

In the non-streaming setting, robust linear regression can be formulated in the following way.
For a system of linear equations $A\xb = {\bf b}$, suppose that a certain fraction of entries in ${\bf b}$ are replaced with arbitrary values so that we have $A\xb = {\bf \tilde{b}} = {\bf b} + {\bf e}$ for some error vector ${\bf e}$ instead. 
Recently, Haddock et al. \cite{haddock2022quantile} have proposed two methods based on the quantile estimate of residual of the current iterate, one is based on the randomized Kaczmarz algorithm, and the other is based on the $\ell_1$ loss function which is known to be robust to outliers in general \cite{huber1987place, koenker1978regression}.  

 Although these methods are effective in solving corrupted systems of linear equations, they may not be suitable for the time-sensitive streaming setting, due to the possible computational burden in estimating quantiles. Possible ways to alleviate these time/memory requirements for estimation include using a sliding window \cite{haddock2022quantile}, approximate quantile estimation \cite{haddock2023subsampled}, and in more general SGD context, choosing the data point with the lowest loss value after observing a sufficiently large number of data points per iteration \cite{shah2020choosing}. However, none of these methods would apply to the fully streaming scenario, where accessing any past data is not possible. At the same time, the corruption model considered in these works might be too strong for the streaming setting. Essentially, it implies that in the worst case a presumed adversary can examine all past equations to choose the worst possible placement of corrupted measurements. 
 Here, we consider a semi-random Massart noise model as described below.

\subsection{On corruption models} Let us recall several main models of adversarial corruption here. By \emph{adversarial corruption model} we imply that the corruptions do not stem from a particular distribution but can be added in arbitrary way. Thus we do not assume an existence of an actual adversary (although it might be a convenient to describe a model and in some of the applications) but rather that the proposed methods are expected to work in the worst case.

In the \emph{fully adversarial corruption model}, as in QuantileSGD and QuantileRK methods in \cite{haddock2022quantile}, before we run the methods, the adversary can select any measurements and replace them with any values (so it can depend on the true parameter and measurement vectors) as long as the corruption fraction is $p$. This effectively makes the setting non-streaming for the adversary. For this noise model, the question about how much of the corruption fraction $p$ can be tolerated by SGD-exp or any reasonable streaming learning method appears to be still open. Even stronger adversarial models such as those considered in \cite{diakonikolas2022streaming}, allow for the adversarial to modify the measurement vectors as well. 

In the \emph{Massart noise model} (this work, also \cite{diakonikolas2021relu}), the adversary cannot choose which measurements to corrupt. Instead, the measurement is randomly selected for corruption with probability $p$. Then, the adversary replaces it with any values (so again, it can depend on the true parameter and the associated parameter vectors). For example, in the streaming setting, with probability $p$, an adversary can inspect a data point at each time and replace the associated label with any incorrect label to confuse the receiver as much as possible. This includes the sign-flip corruption. 

Lastly, in the $\emph{oblivious response corruption}$ model (e.g., \cite{diakonikolas2021relu, pesme2020online}), there is essentially no adversary. Each measurement is randomly selected with a probability of $p$ and then corrupted by additive random noise that is $\emph{independent}$ with the associated measurement vector and the true parameter. This is the weakest corruption model used in works such as \cite{pesme2020online}. This excludes noise types such as sign-flip noise, which is covered in our outlier model, the Massart noise model.

\subsection{Contribution summary} Here, we propose the SGD-exp method to solve the $\ell_1$ minimization problem efficiently in the streaming setting. Since in our streaming setting each data point arrives one at a time, the most generic yet natural corruption model is the Massart noise as described above. As such, we do not require the observations to be obliviously corrupted nor the magnitude or moments of corruption to be bounded.
This is an informal version of our theoretical results, Theorems~\ref{thm:main_convergence} and Theorem~\ref{thm:main_convergence_relu}:
\begin{thm} Let $\xb \in \R^d$ be an unknown vector, let $y_j = f(\langle \xb, \ab_j\rangle) + \epsilon_j$ for $j = 1, 2, \ldots$ represent streaming measurements of $\xb$. Here, $\epsilon_j$ is the Massart noise with corruption probability $p < 0.5$, the measurement vectors $\ab_j$ satisfy Gaussian-like assumptions and $f$ can be an identity or a ReLU function. There exists a constant $\xi_0(p,T) >0$ such that for any $\xi \in (0, \xi_0)$, if we take $\lambda = 1 + \xi$ and run $T$ iterations of $l_1$-SGD 
\begin{equation}\label{sgd-exp}
\xb_{k+1} = \xb_{k} + G \lambda^{-k} \text{sign} ( y_k - f(\langle \xb, \ab_k\rangle)  ) f'(\langle \xb, \ab_k\rangle)  \ab_k,
\end{equation}
where $f'$ is a subgradient of $f$ 
for $T$ iterations, then, with high probability, we have
$$
\|\xb - \xb_T\|_2 \lesssim \log T \exp \left(-\frac{T}{d \log^2 T} \right).$$
\end{thm}

The proof of main theorems relies on a type of analysis that is, to the best of our knowledge, new in the robust regression literature. It is based on drift analysis of the stochastic process after we have transformed the residual equation of SGD, and can be interesting on its own. 

As a result, the proposed method SGD-exp \eqref{sgd-exp} provides (nearly) linear convergence guarantees for any corruption probability less than $0.5$, which is the best possible for the Massart noise model. Our analysis also reveals that SGD-exp tolerates any corruption probability less than $1$ when the corruption is symmetric oblivious noise\footnote{Actually, the only assumption we need is $\mathbb{P}$(the random outlier is positive) = $\mathbb{P}$(the random outlier is negative).} for the streaming setting, answering a related question in \cite{steinerberger2023quantile}. 

To the best of our knowledge, our approach provides the first nearly linear convergence guarantees for both linear regression and ReLU regression under the adversarial 
Massart noise model (see discussion above and Assumption 1 below for the formal definition) in the fully streaming setting.

\subsection{Related works}
For the linear regression in the streaming environment with oblivious random corruption, Pesme and Flammarion \cite{pesme2020online} have proposed a method (which we will refer to as SGD-root due to the square-root decaying step size scheduling), which is based on the SGD for the $\ell_1$ loss, as well as SGD-exp. Compared to it, SGD-exp comes with a faster convergence rate under a less restricted corruption model. Specifically, running SGD-root for $k$-iterations to recover $d$-dimensional signals only provides $O \left( \sqrt{d/k} \right)$ recovery accuracy, whereas our method guarantees the recovery error of the order of $O\left(\exp \left(-k \over d \log^2 k \right) \right)$. 
In \cite{pesme2020online}, the outliers are generated by the random  $\emph{oblivious response corruption }$ model, which excludes noise types such as sign-flip noise, which are covered in the Massart noise model that we consider. SGD-root also requires the corruption noise to have a finite first absolute moment, whereas we do not require any such condition. 

Diakonikolas et al. study robust linear and ReLU regression under the Massart noise model in \cite{diakonikolas2021relu}, but their approach is not designed for the streaming setting and does not provide an explicit convergence rate in terms of the number of iterations.

In \cite{diakonikolas2023distribution}, Diakonikolas et al. study robust regression for generalized linear models. Although this work considers more general distributions and activation functions, it is not for the streaming setting, and the corruption model is the oblivious response corruption, weaker than the Massart noise.

In \cite{diakonikolas2022streaming}, the authors consider an even more general contamination model where 
the measurement vectors can also be modified by the adversary. Due to this general model assumption, the error does not decay below $O(p)$, where $p$ is the corruption probability, even if we increase the number of iterations. 
Moreover, without any distributional assumption on the measurement vectors, the recovery of the true parameter is information-theoretically impossible in general \cite{diakonikolas2021relu,manurangsi2018computational}. The works \cite{diakonikolas2022learning, wu2023finite} aim to find a parameter to minimize the risk function, assuming that the covariates and contaminated responses are jointly random along with an additional condition on the size of contaminated responses. Neither of these conditions is required for our convergence guarantees.

\subsection{SGD with exponential decay step size}

The exponential step decay scheduling or its variants for SGD are commonly used as a default setting in many popular machine learning software packages, including TensorFlow \cite{abadi2016tensorflow} and PyTorch \cite{paszke2019pytorch}. The last iterates of SGD with these types of step size scheduling have demonstrated excellent empirical performance, as observed in \cite{ge2019step, li2021second, wang2023convergence}. However, the first convergence results for exponential step decay scheduling have emerged only recently \cite{li2021second, wang2021convergence}. These studies primarily focus on the problem from an optimization perspective and do not address robustness, thus failing to provide meaningful recovery in cases where measurements or stochastic gradients are contaminated by outliers like our setting.
Our method, SGD-exp, employs this more practical step size scheduling to address robust linear and ReLU regression problems in a streaming setting. 

\subsection{Organization} In the next section, we formalize the streaming measurement model and the assumptions on the noise and on the measurements. Then, we define the method both in linear \eqref{eq:main_iteration} and in ReLU cases \eqref{eq:main_iteration-relu} and state our main convergence theorems, Theorem~\ref{thm:main_convergence} and Theorem~\ref{thm:main_convergence_relu}. In Section~\ref{sec:drift}, we provide necessary background on discrete drift analysis: we state and prove Theorem~\ref{thm:MGF_bound} -- a convenient modification of [\cite{hajek1982hitting}, Theorem 2.3] -- which might be applicable more broadly. In Sections~\ref{sec:proofs} and~\ref{sec:proofs-relu}, we prove our main theorems, and provide empirical evidence in Section~\ref{sec:experiments} to support our theory. We summarize our results and discuss future research directions in Section \ref{sec:conclusion}.

\section{Main results}\label{sec:results}
We formalize the models and state our main results in this section.
 \subsection{Model 1: Streaming linear system}
Suppose we want to recover an unknown vector ${\bf x} \in \mathbb{R}^d$ from $n$ random linear measurements with corruption probability $p$.
More precisely, we have observations $y_1, y_2, \dots, y_n$ arriving in a streaming fashion: 
\begin{equation}\label{eq:measurements}
y_j = \inner{{\bf x}, \ab_j} + \epsilon _j, \quad \text{for} \quad j = 1, \ldots, n,
\end{equation}
such that the noise $\epsilon _j$ satisfies Assumption~\ref{m-noise} and measurement vectors $ {\bf a}_j  \in \R^d$ satisfy Assumption~\ref{assumption:measurement_model}: 
\begin{assumption}[Massart noise model]\mbox{}\\
\label{m-noise} The coordinates of an $n$-dimensional noise vector $\boldsymbol{\epsilon }$ satisfy $\epsilon _j = \xi_j \nu_j$, where  $\xi_j$ are the indicator random variables taking value $1$ with probability $p < 1/2$ and independent with all other variables, and $\nu_j$ are any variables, possibly random and dependent on measurement vectors or true signal. \end{assumption}
\begin{assumption}[Measurement model]\mbox{}\\
\label{assumption:measurement_model} For any $j = 0,1,\ldots$ let $\mathcal{F}_j$ be the $\sigma$-algebra generated by $\{{\ab}_0, \epsilon _0\}, \{{\ab}_1, \epsilon _1\}, \dots, \{{\ab}_{j-1}, \epsilon _{j-1}\}$. 
Let ${\bf a}_j  \in \mathbb{R}^d$ be the unit-norm measurement vectors independent with $\mathcal{F}_j$ such that $\sqrt{d} {\bf a}_j $ are i.i.d. mean-zero isotropic random vectors. For any $\ub \in \R^d$ measurable with respect to $\mathcal{F}_j$ it holds that
\[
\mathbb{E}_{{\bf a}_j } \bigg[  | \inner{\ub, {\bf a}_j }|  \Big|  \ub \bigg] \ge \widetilde{C} {\|\ub\|_2 \over \sqrt{d}}
\] for some constant $\widetilde{C} > 0$ that only depends on the distribution of ${\bf a}_j$.
\end{assumption}

Direct computation shows that the normalized random Gaussian vector, or equivalently, the random vectors drawn uniformly at random from $\mathbb{S}^{d-1}$, satisfy Assumption \ref{assumption:measurement_model} with the constant $\widetilde{C} = \sqrt{2/\pi}$. More generally, many other measurement models such as the normalized Bernoulli random vector satisfy the above assumption for moderately large dimension $d$ by the following lemma. 

\begin{lem}
Let $\{\phib_j\}_{j=1}^n$ be independent random vectors whose entries are i.i.d.,  mean-zero, unit-variance, and sub-Gaussian with sub-Gaussian norm bounded by $K$. Then the normalized vectors $\{\phib_j /\|{\phib_j}\|_2\}_{j=1}^n$ satisfy Assumption \ref{assumption:measurement_model} with a constant $\widetilde{C}$.  Here, $\widetilde{C}$ is a positive constant  that depends on $K$.
\end{lem}

\begin{proof}

Let $\mathcal{F}_j$ be the $\sigma$-algebra generated by $\{{\phib}_0, \epsilon _0\}, \{{\phib}_1, \epsilon _1\}, \dots, \{{\phib}_{j-1}, \epsilon _{j-1}\}$ and $\ub \in \R^d$ be measurable with respect to $\mathcal{F}_j$ as in Assumption \ref{assumption:measurement_model}.
By the Khintchine inequality (see, e.g.,  Exercise 2.6.6 in \cite{vershynin2018high}), 
\[
 \mathbb{E}_{{\phib_j} } \bigg[  | \inner{\ub, {\phib_j}}|  \Big|  \ub \bigg] \ge C(K) {\|\ub\|_2  }.
\]
Also, by Bernstein's inequality, for a given constant $c > 1$, $$\mathbb{P}(\|\phib_j\|_2 > \sqrt{cd}) = \mathbb{P}(\|\phib_j\|_2^2 > cd) \le e^{-c'(K)d}$$ for some constant $ c'(K)$ that only depends on the sub-Gaussian norm of the entries of random vector $\phi_j$. Hence, 
\begin{align*}
   \mathbb{E}_{\phib_j} \bigg[  | \inner{\ub, \frac{\phib_j}{\|\phib_j\|_2}}|  \Big| \ub \bigg] 
    &\ge    \mathbb{E}_{{\phib_j} } \bigg[  | \inner{\ub, \frac{\phib_j}{\|{\phib_j}\|_2}}| \mathbbm{1}_{ \|\phib_j\|_2 \le \sqrt{cd}}   \Big|  u\bigg] \\
   &\ge {1 \over \sqrt{cd}} \mathbb{E}_{{\phib_j} } \bigg[  | \inner{\ub, {\phib_j}} |\mathbbm{1}_{ \|\phib_j\|_2 \le \sqrt{cd}}   \Big|  \ub\bigg]
   \\
   &= {1 \over \sqrt{cd}} \left( \mathbb{E}_{{\phib_j} } \bigg[  | \inner{\ub, {\phib_j}} | \Big|  \ub\bigg]  - \mathbb{E}_{{\phib_j} } \bigg[  | \inner{\ub, {\phib_j}} |\mathbbm{1}_{ \|\phi_j\|_2 > \sqrt{cd}}   \Big|  \ub\bigg] \right) \\
   &\ge {1 \over \sqrt{cd}} \left( \mathbb{E}_{{\phib_j} } \bigg[  | \inner{\ub, {\phib_j}} | \Big|  \ub\bigg]  - \|\ub\|_2  \mathbb{E}_{{\phib_j} } \bigg[  \|{\phib_j}\|_2 \mathbbm{1}_{ \|\phi_j\|_2 > \sqrt{cd}}   \Big|  \ub\bigg] \right) \\
   &\ge {1 \over \sqrt{cd}} \left( C(K) \|\ub\|_2  - \|\ub\|_2 \left( \mathbb{E}_{{\phib_j} } \bigg[  \|{\phib_j}\|_2^2 \bigg] \right)^{1/2} \left(\mathbb{E}_{{\phib_j} } \bigg[ \mathbbm{1}_{ \|\phib_j\|_2 > \sqrt{cd}}   \Big|  \ub\bigg]\right)^{1/2} \right) \\
   &= {1 \over \sqrt{cd}} \left( C(K) \|\ub\|_2  - \|\ub\|_2 \left( \mathbb{E}_{{\phib_j} } \bigg[  \|{\phib_j}\|_2^2 \bigg] \right)^{1/2} \left(\mathbb{E}_{{\phib_j} } \bigg[ \mathbbm{1}_{ \|\phib_j\|_2 > \sqrt{cd}}  \bigg]\right)^{1/2} \right) \\
    &= {1 \over \sqrt{cd}} \left( C(K) \|\ub\|_2  - \sqrt{d} \|\ub\|_2  \mathbb{P}(\|\phib_j\|_2 > \sqrt{cd})^{1/2} \right)\\
    &\ge {\|\ub\|_2 \over \sqrt{cd}} \left( C(K)   -  \sqrt{d} e^{-c'(K)d/2} \right),
\end{align*}
where we have used that  $\ub$ is measurable with respect to $\mathcal{F}_j$ and $\phi_j$ is independent with $\mathcal{F}_j$.
Setting $c = 3/2$ and taking $d$ large enough such that $\sqrt{d} e^{-c'(K)d/2} < C(K)/2$ proves the lemma.
\end{proof}

\begin{rem}[General covariance case]
\end{rem} \vspace{-2mm}
\textit{
\begin{enumerate}
    \item \emph{(Known Covariance Scenario)} The analyses for some works, including SGD-root \cite{pesme2020online} also assume a mean-zero Gaussian feature (measurement), which is more restricted than our measurement model in Assumption \ref{assumption:measurement_model}. Although they cover a non-identity covariance matrix $\Sigma$ for Gaussian vectors, in the case when the covariance matrix is known, we would like to point out that by multiplying by $L^{-1}$ (where $\Sigma = LL^T$ is the Cholesky Decomposition), the vectors can be standardized. Hence, we can  extend our convergence analysis to the non-identity covariance (non-isotropic) Gaussian case. For simplicity, we have carried out our analysis for the isotropic case but under a broader class of random measurement models. 
    \item{\emph{(Unknown Covariance Scenario)}} In some practical settings, $\Sigma$ may be unknown and high-dimensional, making accurate estimation challenging and expensive. While offline robust covariance estimation methods (e.g., \cite{diakonikolas2019robust} or \cite{diakonikolas2022streaming}) can in principle provide an estimate of $\Sigma$, they are not designed to simultaneously handle streaming parameter estimation and may demand considerable amounts of samples or computational resources. Although recent work \cite{diakonikolas2022streaming} has made progress on streaming covariance estimation, it remains nontrivial to integrate such methods into a streaming algorithm for regression based on the estimation of the covariance at each iteration.
\end{enumerate}}

 \subsection{Model 2: Streaming ReLU regression}

Suppose we observe a signal $\xb$ through nonlinear measurements 
$y_i$ given as 

\begin{equation}\label{eq:measurements-relu}
    y_i = \sigma( \inner{\xb, \ab_i}) + \epsilon _i, \quad \text{ for } i = 1, \ldots, n,
\end{equation}
where the ReLU activation function $\sigma(u) := \max \{0, u\}$, the noise $\epsilon _j$ satisfies Assumption~\ref{m-noise} and measurement vectors $ {\bf a}_j  \in \R^d$ satisfy Assumption \ref{assumption:measurement_model} with the following additional condition on the distribution of ${\bf a}$: 

\begin{assumption}
\label{assumption:measurement_model_ReLU}
For any $\ub \in \R^d$ and $\vb \in \R^d$, for any $ i = 1, \ldots$ we have
$$
\E_{\ab_i} \left[\mathbbm{1}_{ \{\inner{\vb, \ab_i} \ge 0\}} \right] = \E_{\ab_i} \left[\mathbbm{1}_{ \{\inner{\vb, \ab_i} \le 0\}} \right]
$$
and
$$
\E_{\ab_i} \left[|\inner{\ub, \ab_i}| \mathbbm{1}_{ \{\inner{\vb, \ab_i} \ge 0\}} \right] = \E_{\ab_i} \left[|\inner{\ub, \ab_i}| \mathbbm{1}_{ \{\inner{\vb, \ab_i} \le 0\}} \right].
$$
\end{assumption}
\begin{rem}
    Note that Assumption \ref{assumption:measurement_model_ReLU} holds for any symmetric distribution for $\ab_i$; the distribution of $\ab_i$ is identical to the distribution of $-\ab_i$. This includes common measurement models such as the Gaussian distribution or the symmetric Bernoulli distribution.
\end{rem}
\begin{rem}
    At least some form of assumption on the distribution of $\ab$ is necessary for ReLU regression, otherwise recovering the unknown vector $\xb$ is information-theoretically impossible \cite{diakonikolas2021relu, manurangsi2018computational}. 
    The Assumption \ref{assumption:measurement_model_ReLU} is similar to the one in \cite{wu2023finite}, but their work aims to minimize the risk with respect to the $\ell_2$ loss and is not about the true parameter recovery using the more robust $\ell_1$ loss as ours. Moreover, unlike our work, their corruption model and result do not allow arbitrary large corruptions. 
\end{rem}

\subsection{SGD-exp method and main theorems}
Consider the following version of stochastic gradient descent for the $\ell_1$-loss, or least absolute deviation error. The same version was considered in QuantileSGD \cite{haddock2022quantile} in the non-streaming setting but with a different step size scheduling. Namely, in the linear case, we define SGD-exp method iteration as
\begin{align}
\label{eq:main_iteration}
{\bf x}_{k+1} = {\bf x}_{k} + G \lambda^{-k} \text{sign} (y_k - \inner{{\bf x}_k, {\bf a}_k} ){\bf a}_k,
\end{align} where $\lambda^{-k}$ is the step size in $k$-th iteration and $\lambda > 1$. The initial iterate ${\bf x}_0$ is set to $0$. Figure~\ref{fig:SGD_corrupted_system} (left) shows that the method empirically converges linearly in the number of iterations for $\lambda = 1.00003$.

Our main result is that 

\begin{thm}
\label{thm:main_convergence}
Suppose that ${\bf x} \in \mathbb{R}^d$ is observed through the noisy streaming scheme \eqref{eq:measurements} 
and the measurement vectors $ {\bf a}_j  \in \R^d$ satisfy Assumption \ref{assumption:measurement_model}. Suppose that the dimension $d$ is sufficiently large enough to satisfy ${ \widetilde{C}(1-2p) \over \sqrt{d}} <  {3 \over 7}$.
Let $R$ be any constant with $R >225$ and let the parameter $G > 0 $ be such that for 
$$\lambda := \sqrt{1 + \widetilde{C}^2  { (1-2p)^2 \over {R d \log^2 T}}},\ \text{ we have } \; {\|\xb\|_2^2 \over G^2} <  {1 \over {2(\lambda^2-1)}}.  
$$ 
Let $\xb_T$ be the result of $T$ steps of SGD as per \eqref{eq:main_iteration}. Then, with probability 

$1 - 70dT^{1 - \sqrt{ R }/15}/\widetilde{C}^2(1-2p)^2 $
 the error is bounded by
	$$
	\|\xb - \xb_T\|_2 
\le G {2\widetilde{C}  \sqrt{Rd} \log T \over 1-2p}   \exp{ \Big \{- T \cdot {\widetilde{C}^2 (1-2p)^2   \over 3{R}  d \log^2 T} \Big \} },
$$

 where the positive constant $\widetilde{C}$ is from Assumption \ref{assumption:measurement_model} on the measurement and it only depends on the  the sub-Gaussian norm $K$.
\end{thm}

\begin{rem} (On the possible choice of the initial step size $G$) 
We note that an exponentially decaying step size schedule puts a natural constraint on how far the iterates can move from ${\bf x}_0$. However, we can set the initial step size, $G$, based on the upper bound of the true solution $\|{\bf x}\|_2$, which is typically assumed to be known in the theoretical analysis of streaming algorithms \cite{dasnear}. For a given constant $R > 225$ that determines $\lambda$ (along with the corruption probability, dimension, and the number of iterations), one way to choose the parameter $G$ is to set $G$ such that $G \ge \|\xb\| \sqrt{2(\lambda^2-1)}$, as stated in Theorem \ref{thm:main_convergence}. Essentially, the initial step size $G$ should be set to be proportional to the upper bound of the true solution $\|{\bf x}\|_2$ if one wants to make the convergence rate independent of the upper bound of the true solution $\|{\bf x}\|_2$.
\end{rem}
\begin{rem}[On the choice of $\lambda$]  Theorem~\ref{thm:main_convergence} also says that the convergence is guaranteed as soon as we choose $\lambda = 1 + \xi$ with positive $\xi$ that is small enough. Choosing $\xi$ smaller than necessary leads to a slower convergence rate, as numerically observed in Figure~\ref{fig:SGD-exp_streaming_various_p_step_sizes}. 
\end{rem}

In the ReLU case, using the subdifferentials of the absolute function and the activation function $\sigma$, the corresponding SGD-exp iteration is given by
\begin{equation}\label{eq:main_iteration-relu}
\xb_{k+1} = \xb_{k} + G \lambda^{-k} \text{sign} ( y_k - \sigma( \inner{\xb_k, \ab_k}) ) \mathbbm{1}_{ \{\inner{\xb_k, \ab_k} \ge 0\} } \ab_k.
\end{equation} 
Figure~\ref{fig:SGD_corrupted_system} (right) shows that SGD-exp empirically converges linearly in the number of iterations for $\lambda = 1.00003$ for the streaming with corrupted ReLU measurements.
\begin{figure}[h]
\centering
   \includegraphics[width=0.45 \textwidth]{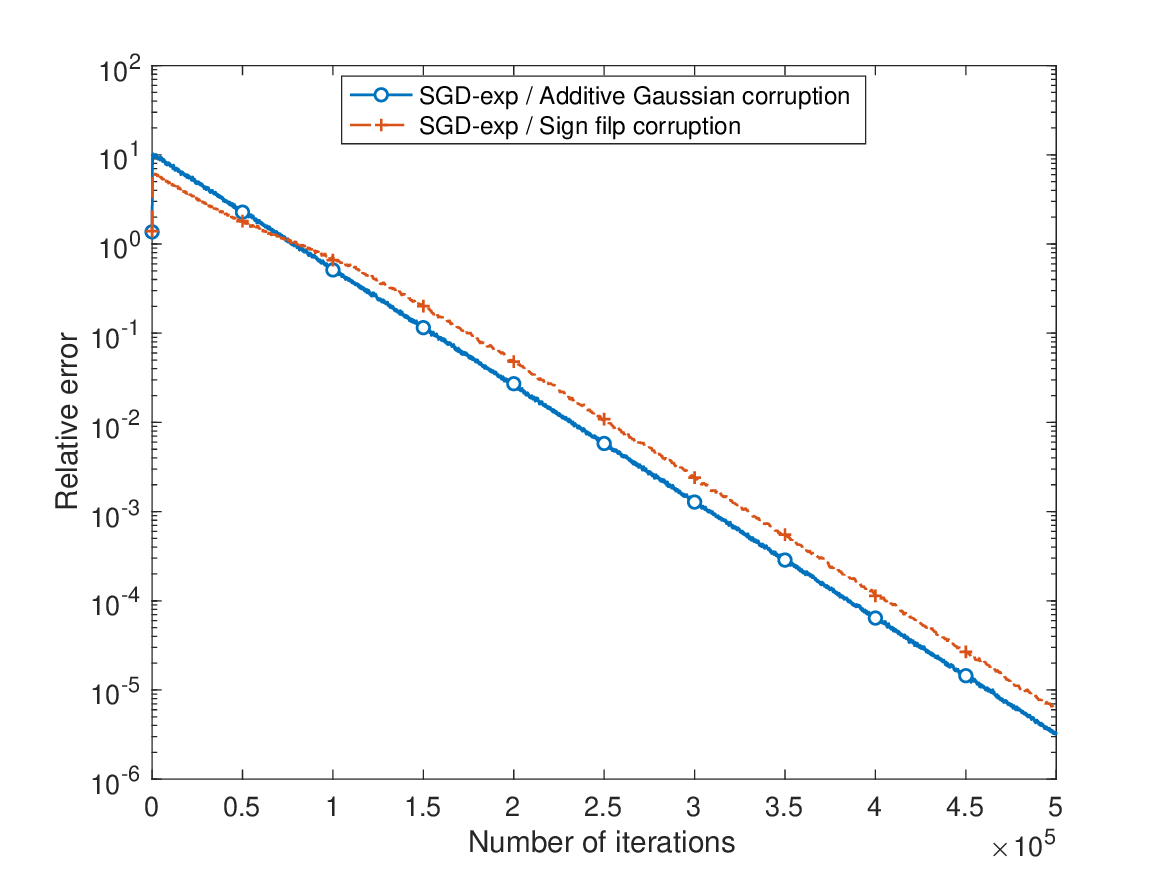}
   \includegraphics[width=0.45 \textwidth]{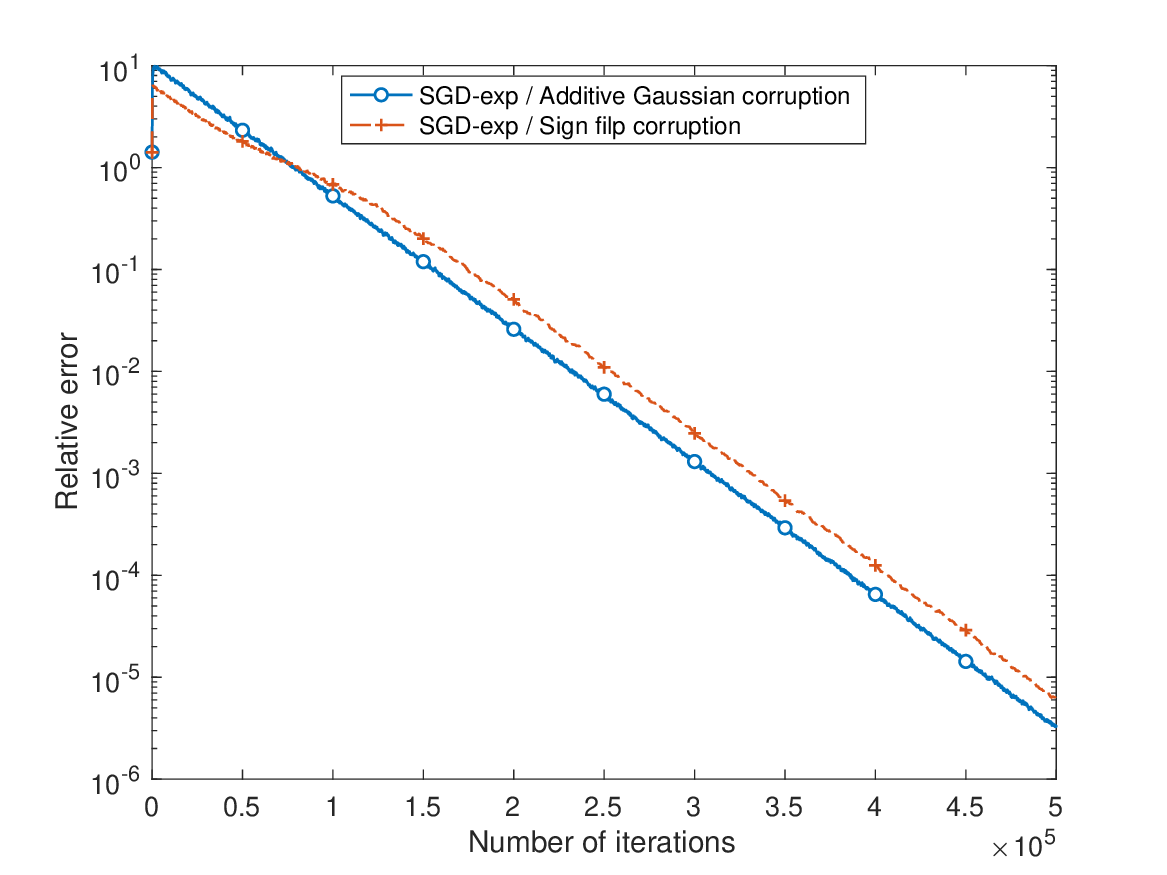}

		\caption[Relative error of the method] {Relative error of the SGD-exp for (a) a corrupted linear system (left) and (b) a corrupted rectified linear (ReLU) system (right) with corruption probability $p = 0.4$.  The blue curves represent relative error when the corruption is a large additive Gaussian noise. The red curves represent the sign-flip corruption, i.e., with probability $p$, the sign of measurement is flipped. The measurement vectors are $100$-dimensional i.i.d. normalized standard Gaussian vectors. The dimension of the signal is $100$ and both plots are averaged over $20$ trials. 
		}
		\label{fig:SGD_corrupted_system}
	\end{figure}
 
The main theorem in the ReLU case is as follows:
\begin{thm}
\label{thm:main_convergence_relu}
Suppose ${\bf x} \in \mathbb{R}^d$ is observed through the noisy ReLU streaming scheme \eqref{eq:measurements-relu} 
 and the measurement vectors $ {\bf a}_j  \in \R^d$ satisfy Assumptions~\ref{assumption:measurement_model} and \ref{assumption:measurement_model_ReLU}. Suppose that the dimension $d$ is sufficiently large enough to satisfy $ { \widetilde{C}(1-2p) \over \sqrt{d}} < 1$. 
  Let $R$ be any constant with $R > 400$ and let the parameter $G > 0 $ be such that for 
$$\lambda := \sqrt{1 + \widetilde{C}^2  { (1-2p)^2 \over {R d \log^2 T}}},\ \text{ we have } \; {\|\xb\|_2^2 \over G^2} < {1 \over {2(\lambda^2-1)}}.   
 $$ Let $\xb_T$ be the result of $n$ steps of SGD-exp as per \eqref{eq:main_iteration-relu}. Then, with probability $$1 -  {120 d  \over \widetilde{C}^2(1-2p)^2 } \cdot  T^{1 - \sqrt{ R }/20} $$

 the error is bounded by

	$$
	\|\xb - \xb_T\|_2 
\le G {2\widetilde{C}  \sqrt{Rd} \log T \over 1-2p}   \exp{ \Big \{- T \cdot {\widetilde{C}^2 (1-2p)^2   \over 3{R}  d \log^2 T} \Big \} },
$$
 where the positive constant $\widetilde{C}$ is from Assumption \ref{assumption:measurement_model} on the measurement and it only depends on the sub-Gaussian norm $K$.
\end{thm}

\begin{rem}[Addressing highly corrupted regime]
Theorem \ref{thm:main_convergence} implies that the iterates of the SGD-exp converge (almost) linearly up to logarithmic factors for the corruption probability $p < 1/2$. This is theoretically the best possible: suppose that $p > 1/2$. Then, any approach would not be able to determine whether the responses are generated from the signal $x$ with the sign-flip corruption  with probability $p$ or they are generated from $-x$ with the sign-flip corruption with probability $p-1/2$. Prominently, in the case when the noise is symmetric and oblivious, our proof yields the result for all $p < 1$, see Remark~\ref{rem:symmetric oblivious corruption} and Remark~\ref{rem:symmetric oblivious corruption_ReLU} for the details.
\end{rem}

\begin{rem}[Convergence rate is almost like in an uncorrupted setting]
 
Our effective convergence rate from Theorem~\ref{thm:main_convergence} and Theorem~\ref{thm:main_convergence_relu} is $$\exp \left(-C{ T \over d} {(1-2p)^2 \over \log^2 T} + \log \left(\frac{\sqrt{d}\log{T}}{1 - 2p} \right)\right).$$ Notably, it is only logarithmically slower than compatible streaming algorithms on the uncorrupted setting. For example, in the streaming setting, the standard Kaczmarz convergence rate is \\
$\exp \left( -C{T \over d} \right)$ for $d$-dimensional Gaussian or Bernoulli measurement vectors. Then, in \cite{li2021second}, an SGD with similar exponential step size schedule $\lambda \sim T^{(1/T)}$ results in the loss decreases of order $O(\exp(-C(d, L)\frac{T}{\log T}))$ on the setting without contamination and under appropriate smoothness $L$-Lipschitz conditions.

\end{rem}

\begin{rem}[SGD-exp is still finite horizon]
While SGD-exp operates in an online manner, requiring access to one data point at a time without revisiting past data, it does depend on prior knowledge of the total number of iterations $T$ to set the exponential decay factor $\lambda$ for the step size scheduling. 
Although there are some works for streaming methods that require the horizon length $(T)$ in advance such as \cite{tsai2022heavy, dasnear} like ours, the dependence on $T$ for the step size schedule may limit its applicability in fully streaming settings. In practice, the horizon length $T$ may often be estimated based on domain knowledge, such as desired precision or a sample budget specific to the ML or statistics application. Exploring adaptive step size schedules that do not depend on $T$ for SGD-exp is an interesting direction for future work.
\end{rem}

The idea of convergence analysis  in both Theorem~\ref{thm:main_convergence} and Theorem~\ref{thm:main_convergence_relu} is to show that a sequence of random variables $\|\ub_T\| := \lambda^T \|\xb - \xb_{T}\|/G$ is (almost) uniformly bounded with high probability, which in turn shows the residual error $\|\xb - \xb_{T}\|$ decreases (almost) geometrically since $\lambda > 1$. To this end, we adapt for our case the result from Hajek \cite{hajek1982hitting} on the drift analysis of discrete stochastic processes, also used in several stochastic algorithm convergence analyses recently \cite{akimoto2022global, bansal2020line}.

\section{Background on drift analysis}\label{sec:drift}

Consider a sequence of random variables $Y_0, Y_1, \dots, Y_n$ measurable with respect to a filtration $\{\mathcal{F}_k\}_{k\ge 0}$ (that is, $Y_0, Y_1, \dots, Y_k$ are $\mathcal{F}_k$-measurable). We follow the general convention that $\mathcal{F}_0 = \emptyset$. Let $\tau_b$ be the first hitting time defined as
$$\tau_b := \min \{j: Y_j \ge b \}.$$
The goal is to obtain a high probability upper bound on the first hitting time defined above, given that the random process's increments are bounded in expectation. The next theorem is a variant of Theorem 2.3 in \cite{hajek1982hitting}, modified to suit our specific purpose.

\begin{thm}
\label{thm:MGF_bound}
Let $\{Y_k\}_{k \ge 0}$ be a discrete stochastic process adapted to the filtration $\{\mathcal{F}_k\}_{k\ge 0}$, where $\mathcal{F}_0$ is a trivial $\sigma$-algebra. Suppose that for some $a,b, \eta, \rho, D \in \R$ such that $a < b$, $\eta> 0$,  $0 < \rho <1$ and $D \ge 1$, the process satisfies 
\begin{itemize}
\item[(C0)] $Y_0 \in [0,a)$ a.s.,
\item[(C1)] $ \mathbb{E} [e^{\eta (Y_{k+1} - Y_{k})} \mathbbm{1}_{ \{a \le Y_k < b\} } | \mathcal{F}_k] \le \rho $ 
a.s. for any $k \ge 1$,
\item[(C2)] $ \mathbb{E} [e^{\eta (Y_{k+1} - a)} \mathbbm{1}_{ \{Y_k < a\} } | \mathcal{F}_k] \le D $
 
a.s. for any $k \ge 1$.
\end{itemize}
Then, for the first hitting time $\tau_b := \min \{j: Y_j \ge b \}$ it hold for any $k \ge 1$
\begin{equation}\label{eq:exp_bound}
\mathbb{E} \Big[ e^{\eta Y_{k}} \mathbbm{1}_{ \{\tau_b >k-1\} }  \Big]   \le  e^{\eta a} \left(\rho^k  + { {1 - \rho^k} \over {1 - \rho} } D \right).
\end{equation} 

\end{thm}

\begin{proof} First, we note that $\tau_b$ is indeed a stopping time, namely, 
$\{\tau_b >l\} = \{Y_1 <b, \ldots, Y_l < b\} \in \mathcal{F}_l$ for any $l \ge 0$. Hence, we have a recursive estimate
\begin{align}\label{eq:recursive_bound}
 \mathbb{E} [ e^{\eta Y_{l+1}} \mathbbm{1}_{ \{\tau_b > l\} } | \mathcal{F}_l] 
 &= \mathbb{E} [e^{\eta Y_{l+1} } \mathbbm{1}_{ \{\tau_b > l\} }  \mathbbm{1}_{ \{Y_l < a\} } | \mathcal{F}_l] +\mathbb{E} [e^{\eta Y_{l+1} } \mathbbm{1}_{ \{\tau_b > l\} }  \mathbbm{1}_{ \{b > Y_l \ge a\} } |  \mathcal{F}_l]  \nonumber\\ 
 &= \mathbb{E} [e^{\eta Y_{l+1} }\mathbbm{1}_{ \{\tau_b > l-1\} }  \mathbbm{1}_{ \{Y_l < a\} } | \mathcal{F}_l] +\mathbb{E} [e^{\eta Y_{l+1} } \mathbbm{1}_{ \{\tau_b >l-1\} }  \mathbbm{1}_{ \{b > Y_l \ge a\} } |  \mathcal{F}_l] \nonumber \\ 
 &=  \mathbbm{1}_{ \{\tau_b >l-1\} }  \mathbb{E} [e^{\eta Y_{l+1} }  \mathbbm{1}_{ \{Y_l < a\} } | \mathcal{F}_l] + \mathbbm{1}_{ \{\tau_b >l-1\} } \mathbb{E} [e^{\eta Y_{l+1} }  \mathbbm{1}_{ \{b > Y_l \ge a\} } |  \mathcal{F}_l] \nonumber \\ 
 &\le 
 \rho e^{\eta Y_l } \mathbbm{1}_{ \{\tau_b >l-1\} }   +   D e^{\eta a},
\end{align}
where in the second step we have used that $\tau_b \ne l$ if $Y_l < b$ and (C1), (C2) in the last step.

Now, since $Y_0 < a$ a.s. and $\mathcal{F}_0$ is a trivial $\sigma$-algebra, \eqref{eq:exp_bound} will immediately follow if we prove that for every $m \le k$
\begin{equation}\label{eq:exp_bound_conditioned}
\mathbb{E} \Big[ e^{\eta Y_{k}} \mathbbm{1}_{ \{\tau_b >k-1\} }  \Big| \mathcal{F}_{k-m} \Big]    \le \rho^m  e^{\eta Y_{k-m} } \mathbbm{1}_{ \{\tau_b > k-m-1\} } + \sum_{i=0}^{m-1} \rho^iD e^{\eta a}.
\end{equation}
For $m = 0$, \eqref{eq:exp_bound_conditioned} holds trivially. Let's employ induction: suppose \eqref{eq:exp_bound_conditioned} 
 holds for some $m \ge 0$, then,  using the tower property of conditional expectation we have for $m+1$:
\begin{align*}
\mathbb{E} \Big[ e^{\eta Y_{k}} \mathbbm{1}_{ \{\tau_b >k-1\} }  \Big| \mathcal{F}_{k-m-1} \Big]  &= \mathbb{E} \Big[\mathbb{E} \Big[ e^{\eta Y_{k}} \mathbbm{1}_{ \{\tau_b >k-1\} }  \Big| \mathcal{F}_{k-m} \Big]\Big| \mathcal{F}_{k-m-1} \Big]  \\
&\le \rho^m  \mathbb{E} \Big[e^{\eta Y_{k-m} } \mathbbm{1}_{ \{\tau_b > k-m-1\} }\Big| \mathcal{F}_{k-m-1} \Big] + \sum_{i=0}^{m-1} \rho^iD e^{\eta a} \\
&\le \rho^{m+1} e^{\eta Y_{k-m-1} } \mathbbm{1}_{ \{\tau_b >k-m-2\} } +  \rho^m D e^{\eta a} + \sum_{i=0}^{m-1} \rho^iD e^{\eta a},
\end{align*}
where we have used \eqref{eq:recursive_bound} with $l = k-m-1$ in the last step. This concludes the proof of \eqref{eq:exp_bound_conditioned} and of Theorem \ref{thm:MGF_bound}.

\end{proof}

We will use the result about the drift through the following key corollary that is a modification of Proposition 2.5 in \cite{hajek1982hitting}: 
\begin{cor}
\label{cor:stopping_time_bound}
Suppose $Y_0, Y_1, Y_2, \dots, Y_n$ satisfy following two conditions (C1) and (C2) in Theorem \ref{thm:MGF_bound} for some $b > a > Y_0 \ge 0$ and $\eta> 0$,  $0 < \rho <1$ and $D \ge 1$. Now, let $K$ be a given nonnegative integer. 
Then, we have
\[
\mathbb{P}[ \tau_b \le K  ] \le  { 1 \over {1 - \rho} } K D e^{ -\eta (b - a )}
\] for any $b$ with $b > a > Y_0 \ge 0$.
\end{cor}

\begin{proof}

	The proof follows from  Markov's inequality applied to the result of Theorem~\ref{thm:MGF_bound}:
\begin{align*}
    \mathbb{P}[ \tau_b \le K  ] = \sum_{k=1}^K \mathbb{P}[ \tau_b = k  ] & = \sum_{k=1}^K \mathbb{P}[ \{Y_k \ge b\} \cap \{\tau_b > k -1\} ] \\
    &= \sum_{k=1}^K\mathbb{P}[e^{\eta Y_{k}} \mathbbm{1}_{ \{\tau_b >k-1\} } \ge e^{\eta b}]\\
    &\overset{\text{Markov's}}{\le}
 \sum_{k=1}^K   e^{- \eta (b-a)} \left(\rho^k  + { {1 - \rho^k} \over {1 - \rho} } D \right) \\
 &\le e^{- \eta (b-a)}  \sum_{k=1}^K \left(\frac{D}{1 - \rho} - \rho^k(  \frac{D}{1 - \rho} - 1) \right)\le e^{-\eta (b - a )} \frac{DK}{1-\rho}.
\end{align*} 
The last inequality holds since $\rho^k (\frac{D}{1-\rho} - 1)$ is positive as $D >1$ and $\rho \in (0,1)$. 
\end{proof}

\section{SGD-exp convergence analysis for linear problem}\label{sec:proofs}

Now, we give the estimates of the type (C1) and (C2) as per  Theorem \ref{thm:MGF_bound} for the process $$Y_k = \|{\bf u}_k\|_2^2 \quad \text{ with } \ub_k = \lambda^{k}({\bf x} - {\bf x}_k)/G,$$ where $\bf{x}_k$ are the iterates of SGD, $\bf{x}$ is the solution and $\lambda$ defines the step size as per \eqref{eq:main_iteration} .

\subsection{Initial reductions}
First, let us compute $Y_k$ explicitly. By definition of measurements \eqref{eq:measurements} and linearity of the inner product, we have

$$
\zb_{k+1} = \zb_{k} - G \lambda^{-k} \text{sign} (\inner{\zb_k, \ab_k} + \epsilon _k )\ab_k \quad \text{ for } \zb_k := \xb-\xb_k.
$$

Multiplying $\lambda^{k+1}$ and dividing by $G$ to both sides we have
\[
    \frac{\lambda^{k+1}}{G} \zb_{k+1}
   = \frac{\lambda^{k+1}}{G} \zb_{k}  - \frac{\lambda^{k+1}}{G} \bigl[ G \lambda^{-k} \, \text{sign} \bigl( \langle \zb_{k}, \ab_k \rangle + \epsilon_k \bigr) \ab_k \bigr]
\] and
\begin{align}
\label{eq:main_iteration2}
\ub_{k+1} = \lambda \Big \{ \ub_{k} - \text{sign} \Big (\inner{\ub_k, \ab_k} + {\lambda^k\epsilon _k \over G} \Big)\ab_k \Big \} \quad \text{ for } \ub_k := {\lambda^k \zb_k \over G}
\end{align}
since $\lambda$, $G$ are positive, so $\text{sign}(\inner{\zb_k, \ab_k} + \epsilon _k) = 
\text{sign} \left(\inner{ {\lambda^k \zb_k \over G}, \ab_k} + {\lambda^k\epsilon _k \over G} \right) = \text{sign} \left(\inner{\ub_k, \ab_k} + {\lambda^k\epsilon _k \over G} \right)$. 

Then, we have
\[
\ub_{k+1} = \lambda \Big \{ \ub_{k} - \text{sign} \Big (\inner{\ub_k, \ab_k} + {\lambda^k \epsilon_k /G}   \Big)\ab_k \Big \}.
\]

Next, we take the squared Euclidean norm on both sides above to get

\begin{equation}\label{yk-def}
Y_{k+1} = || \ub_{k+1} ||^2 = \lambda^2 \Big \{ || \ub_{k} ||^2 - 2 \inner{\ub_k, \ab_k} \text{sign} \Big (\inner{\ub_k, \ab_k} + \lambda^k \epsilon_k/G  \Big) + 1 \Big \}
\end{equation}  where we have used the fact that  $||\ab_k|| = 1$.

\subsection{SGD drift estimates}

\begin{lem} 
\label{lem:main_lemma_1}
Let the random process $Y_k$ defined as per \eqref{yk-def}, where $\{ {\ab}_j \}_{j=1}^{n}$ are i.i.d. random vectors in $\R^d$ satisfying Assumption 1 with some constant $\tilde{C} > 0$ and the corruption noise $\epsilon _j$ is non-zero with probability $p < 1/2$. Let $\mathcal{F}_k$ be the $\sigma$-algebra generated by $\{{\ab}_0, \epsilon _0\}, \{{\ab}_1, \epsilon _1\}, \dots, \{{\ab}_{k-1}, \epsilon _{k-1}\}$. Then for any step size decay parameter $\lambda > 0$ that satisfies 
\begin{equation}\label{lam-bound}
1 < \lambda^2 \le  1 +   \widetilde{C}^2{(1-2p)^2 \over 9d} <  \frac{50}{49},
\end{equation}
for $a = \frac{1}{2(\lambda^2 -1)}, \;
b = \frac{3}{2(\lambda^2 - 1)}$ and
$$
\eta = c^*\sqrt{\lambda^2-1} \; \text{ with } \;  c^* = {1 \over 8 \lambda^2}  \left[ {  \sqrt{2}\lambda^2 (1-2p) \widetilde{C}  \over  \sqrt{d}} - \sqrt{\lambda^2-1}({3 \over 2}  + \lambda^2  ) \right]
$$
we have
    \begin{align*}
&\mathbb{E} [e^{\eta (Y_{k+1} - Y_k )} \mathbbm{1}_{ \{a \le Y_k < b\} } | \mathcal{F}_k] \le  1 - {\widetilde{C}^2(1-2p)^2  \over 60 d }.
\end{align*}

\end{lem}

\begin{proof}
Note that each $\ub_{k} = \lambda^{-k}(\xb - \xb_k)/G$ is measurable with respect to $\mathcal{F}_k$ by \eqref{eq:main_iteration2}, 
 so $\{Y_k\} = \{\|{\bf u}_k\|_2^2\}$ is adapted to $\mathcal{F}_k$. From Taylor's series expansion, with $\Delta_k = Y_{k+1} - Y_k$,
 \begin{equation}\label{taylor}
 \mathbb{E} [e^{\eta \Delta_k} \mathbbm{1}_{ \{a \le Y_k < b\} } | \mathcal{F}_k] \\
 \le 1 + \eta \mathbb{E} \Big [\Delta_k \mathbbm{1}_{ \{a \le Y_k < b\} } |  \mathcal{F}_k \Big] + \mathbb{E} \bigg[ \sum\limits_{n=2}^{\infty} {1 \over n!} \eta^n |\Delta_k|^n \mathbbm{1}_{ \{a \le Y_k < b\} } \Big|  \mathcal{F}_k \bigg].
 \end{equation}
 First, we estimate the linear term:
\begin{align}\label{first-term}
\mathbb{E} [ (Y_{k+1} &- Y_k) \mathbbm{1}_{ \{a \le Y_k < b\} } | \mathcal{F}_k ]\nonumber \\ 
& \overset{\eqref{yk-def}}{=} \mathbb{E} \Big [ \left( (\lambda^2-1)||\ub_k||^2 - 2\lambda^2 \inner{\ub_k, \ab_k} \text{sign} (\inner{\ub_k, \ab_k} + \lambda^k \epsilon_k/G )  + \lambda^2 \right) \mathbbm{1}_{ \{a \le \|\ub_k\|^2 < b\} } \Big|  \mathcal{F}_k  \Big ] \nonumber\\
& \stackrel{(i)} \le  \mathbb{E} \bigg[ \big({3 \over 2}  + \lambda^2\big)\mathbbm{1}_{ \{a \le \|\ub_k\|^2 < b\} } - 2 \lambda^2    \inner{\ub_k, \ab_k} \text{sign} \Big (\inner{\ub_k, \ab_k} + \lambda^k \epsilon _k/G \Big) \mathbbm{1}_{ \{a \le \|\ub_k\|^2 < b\} }  \Big|    \ub_k  \bigg]\nonumber\\
&\stackrel{(ii)}{\le}  \big({3 \over 2}  + \lambda^2\big)
- 2 \lambda^2  \mathbb{E} \bigg[  \inner{\ub_k, \ab_k} \Big \{ (1-p) \text{sign} (\inner{\ub_k, \ab_k} ) - p \cdot \text{sign} (\inner{\ub_k, \ab_k} ) \Big \}  \Big|  \ub_k \bigg]\mathbbm{1}_{ \{a \le \|\ub_k\|^2 < b\} }\nonumber\\
&=  \big({3 \over 2}  + \lambda^2\big)
- 2 \lambda^2 (1 - 2p)  \mathbb{E} \bigg[  | \inner{\ub_k, \ab_k}|  \Big|  \ub_k \bigg]\mathbbm{1}_{ \{a < \|\ub_k\|^2 < b\} }\nonumber\\
& \stackrel{(iii)} {\le} 
{3 \over 2}  + \lambda^2-  2\lambda^2 (1-2p) {  \widetilde{C}  \over  \sqrt{d}\sqrt{2(\lambda^2-1)} }  \mathbbm{1}_{ \{a \le Y_k < b\} },
\end{align}

where the inequality $(i)$ follows from independence of $\{\ab_k, \epsilon _k\}$ with $\mathcal{F}_k$ and $\ub_k$ being $\mathcal{F}_k$-measurable, and also $b = 3/2(\lambda^2 - 1)$. 

The inequality $(ii)$ follows from the worst case that flips the $\text{sign} (\inner{\ub_k, \ab_k} )$, which happens if all the $\text{sign} (\inner{\ub_k, \ab_k} )$ flips when outlier occurs. Lastly, the inequality $(iii)$ follows from Assumption \ref{assumption:measurement_model} for the measurement vector $\ab_k$ and $a = 1/2(\lambda^2 -1)$.

Now, for the second term in \eqref{taylor}, we have for $n \ge 2$
\begin{align*}
{1 \over n!} |Y_{k+1} - Y_k|^n \mathbbm{1}_{ \{a \le Y_k < b\} } &\le {1 \over n!} \Big |  (\lambda^2-1)\|\ub_k\|^2 + 
2\lambda^2 | \inner{\ub_k, \ab_k}| + \lambda^2 \Big|^n  \mathbbm{1}_{ \{a \le \|\ub\|^2 < b\} } \\
&\le {1 \over n!} \bigg |  {3 \over 2} + \lambda^2 + {2\lambda^2 \sqrt{3}   \over \sqrt{2(\lambda^2-1)} } \bigg|^n \le {1 \over 2} \bigg |  {3 \over 2} + \lambda^2 + {2\lambda^2 \sqrt{3}   \over \sqrt{2(\lambda^2-1)} } \bigg|^n,
\end{align*}
and, choosing $\eta = c \sqrt{\lambda^2-1}$,
$$\mathbb{E} \bigg[ \sum\limits_{n=2}^{\infty}  {1 \over n!} \eta^n |\Delta_k|^n \mathbbm{1}_{ \{a \le Y_k < b\} } \Big|  \mathcal{F}_k \bigg] \le \sum\limits_{n=2}^{\infty}  
	y^n,  \quad \text{ for } y =  c\sqrt{\lambda^2-1}({3 \over 4} + {\lambda^2 \over 2}) +   
	c{ \sqrt{3} \lambda^2  \over \sqrt{2} } > 0.
$$

Then, if we choose $c$ small enough so that $y < 0.5$, and using that $\lambda^2 < 50/49$ by assumption, the geometric sum can be further upper-bounded as
$$
\sum\limits_{n=2}^{\infty}
	y^n \le \frac{y^2}{1-y} \le 2y^2 \le 2 (1.43\lambda^2c)^2 \le 4\lambda^4c^2.
$$

Using this estimate together with \eqref{first-term} in the initial Taylor expansion \eqref{taylor}, we have 
\begin{align*}
\mathbb{E} [e^{c \sqrt{\lambda^2-1} (Y_{k+1} - Y_k )} \mathbbm{1}_{ \{a \le Y_k < b\} } | \mathcal{F}_k] &
 \le 1 - u c  
+ 4 \lambda^4 c^2, \\
& \text{ where } \quad 
u =  {  \sqrt{2}\lambda^2 (1-2p) \widetilde{C}  \over  \sqrt{d}} - \sqrt{\lambda^2-1}({3 \over 2}  + \lambda^2  ).
\end{align*}
The condition $1 < \lambda^2 < 1 +  \widetilde{C}^2{(1-2p)^2 \over 9d}$  essentially ensures that the first term dominates in $u$.
This quadratic polynomial in $c$ is minimized by
\begin{equation}\label{c-star}c^{*} = {u \over 8 \lambda^4}  > {(\sqrt{2} - 1/3 - 1/2) \over 8 \lambda^2}   {   \widetilde{C} (1-2p)  \over \sqrt{d}} > {1 \over 15} \bigg[  { \widetilde{C} (1-2p) \over \sqrt{d}}   \bigg] \end{equation}
using the estimates on $\lambda$ \eqref{lam-bound} to show the inequalities, which in turn yields

\begin{align*}
&\mathbb{E} [e^{c^*\sqrt{\lambda^2 -1} (Y_{k+1} - Y_k )} \mathbbm{1}_{ \{a < Y_k < b\} } |  \mathcal{F}_k] \le 1 - 4 \lambda^4 (c^*)^2 \le  1 - {\widetilde{C}^2(1-2p)^2  \over 57 d }.
\end{align*}

To conclude the proof of Lemma~\ref{lem:main_lemma_1}, note that $c^*$ is indeed small enough to make $y < 0.5$. This follows by a direct check using that $y < 1.43\lambda^2 c^*$ and the conditions on $\lambda$ \eqref{lam-bound}.
\end{proof}

\begin{lem} 
\label{lem:main_lemma_2}Let the random process $Y_k$ be defined as per \eqref{yk-def}, where $\{ {\ab}_j \}_{j=1}^{n}$ are i.i.d. random vectors in $\R^d$ satisfying Assumption 1 with some constant $\tilde{C} > 0$ and the corruption noise $\epsilon _j$ is non-zero with probability $p < 1/2$. Let $\mathcal{F}_k$ be the $\sigma$-algebra generated by \\
$\{{\ab}_0, \epsilon _0\}, \{{\ab}_1, \epsilon _1\}, \dots, \{{\ab}_{k-1}, \epsilon _{k-1}\}$. Then for any step size decay parameter $\lambda > 0$ that satisfies
\begin{equation*}
1 < \lambda^2 \le  1 +   \widetilde{C}^2{(1-2p)^2 \over 9d} <  \frac{50}{49},
\end{equation*}
for 
$$
a = \frac{1}{2(\lambda^2 -1)}; \quad \eta = c^*\sqrt{\lambda^2-1};  \quad  c^* = {1 \over 8 \lambda^2}  \left[ {  \sqrt{2}\lambda^2 (1-2p) \widetilde{C}  \over  \sqrt{d}} - \sqrt{\lambda^2-1}({3 \over 2}  + \lambda^2  ) \right]
$$
we have
\begin{align*}
&\mathbb{E} [e^{\eta (Y_{k+1} - a )} \mathbbm{1}_{ \{ Y_k \le a\}} |  \mathcal{F}_k] 
\le \exp\Big \{ {\widetilde{C} (1-2p) \over 3 \sqrt{d}} \Big \}. 
\end{align*}

\end{lem}

\begin{proof}
Using the event $\{Y_k \le a \} \subseteq \{\|{\bf u}_k\|^2 \le \frac{1}{2(\lambda^2 -1)}\}$ and that $\|\ab_k\|^2 = 1$, we estimate
\begin{align*} 
\mathbb{E} [e^{\eta (Y_{k+1} - a )} \mathbbm{1}_{ \{ Y_k \le a\}} |   \mathcal{F}_k] 
&\le \mathbb{E} [e^{\eta (Y_{k+1} - Y_k )} \mathbbm{1}_{ \{ Y_k \le a\}} |  \mathcal{F}_k] \\
& \le  \mathbb{E} [\exp(\eta (  (\lambda^2-1)\|\ub_k\|^2 + 
2\lambda^2 | \inner{\ub_k, \ab_k}| + \lambda^2) ) \cdot \mathbbm{1}_{ \{||u_k||^2 \le a\}}|\ub_k] \\
& \le \exp \Big \{ c^*\sqrt{\lambda^2 - 1} \Big( {1 \over 2} + \lambda^2 + {2\lambda^2 \over \sqrt{2(\lambda^2-1)} } \Big)  \Big \} \\
& \le \exp \Big \{  {  \widetilde{C}  (1-2p) \over 5 \sqrt{d}}  \Big(\sqrt{\lambda^2-1}  +  {3 \over 2} \Big) \Big \} \le \exp \Big \{  {  \widetilde{C}  (1-2p) \over 3 \sqrt{d}} \Big \}, 
\end{align*}
where the last line is from the choice of  $c^{*} <  { \sqrt{2} \widetilde{C} (1-2p) \over 8 \lambda^2 \sqrt{d}}$ and $\sqrt{2}/8 < 1/5$.

\end{proof}

\subsection{Proof of Theorem \ref{thm:main_convergence}}

We are ready to give the proof of our main result in the linear case.

\begin{proof} [Proof of Theorem \ref{thm:main_convergence}] 
We apply Corollary \ref{cor:stopping_time_bound} with \begin{align*}
    &\rho = 1 - {\widetilde{C}^2 (1-2p)^2  \over 60 d}, \\
&D =  \exp \Big \{ { \widetilde{C} (1-2p) \over 3 \sqrt{d}}\Big \}, \\
&\eta = c^*\sqrt{\lambda^2-1} \; \text{ with } \;  c^* = {1 \over 8 \lambda^2}  \left[ {  \sqrt{2}\lambda^2 (1-2p) \widetilde{C}  \over  \sqrt{d}} - \sqrt{\lambda^2-1}({3 \over 2}  + \lambda^2  ) \right],\\
&a = \frac{1}{2(\lambda^2 -1)}, \; \text{ and } \;
b = \frac{3}{2(\lambda^2 - 1)}.
\end{align*}
Using that $\eta > {1 \over 15} \bigg[  { \widetilde{C} (1-2p) \over \sqrt{d}}   \bigg] \sqrt{\lambda^2 -1}$ (by equation \eqref{c-star}), definition of $\lambda$ and $ { \widetilde{C} (1-2p) \over {3\sqrt{d} \log T}} < {1 \over 7}$, we get 
\begin{align*}
\mathbb{P}[ \tau_b \le T  ] 
& \le  { 1 \over {1 - \rho} } T D e^{ -\eta (b - a )} \\ 
& \le {60d  \over {\widetilde{C}^2}(1-2p)^2 }  T \exp \Big \{  {\widetilde{C}  (1-2p) \over 3 \sqrt{d}} \Big \} \exp \bigg\{ {-   {\widetilde{C}  (1-2p) \over 15\sqrt{d}} {1 \over \sqrt{\lambda^2-1} }} \bigg \} \\
& \le {70d T  \over \widetilde{C}^2(1-2p)^2 } \cdot  \exp \bigg\{ {-   {1 \over 15 } \cdot \sqrt{ {{R }}} \log T  } \bigg \} =: p_*.
\end{align*}

Recall that $\tau_b := \min \{j: \| \ub_j \|^2   \ge b \} = \min \{j: \lambda^{2j} \|\xb-\xb_j\|_2^2/G^2 \ge b \}.$ So, with probability $1 - p_*,$ 
we have $\tau_b > n$ and 
\begin{align*}
||\xb - \xb_T||_2 & < G \sqrt{b}  \lambda^{-T}\\
&\le G \sqrt{\frac{3}{2(\lambda^2 -1)}} \lambda^{-T} \\
&\le G {2\widetilde{C}  \sqrt{Rd} \log T \over 1-2p}  \left(1 + \widetilde{C}^2 { (1-2p)^2 \over {R d \log^2 T}} \right)^{-T/2}\\
&\le G {2\widetilde{C}  \sqrt{Rd} \log T \over 1-2p}   \exp{ \Big \{- T \cdot {\widetilde{C}^2 (1-2p)^2   \over 3{R}  d \log^2 T} \Big \} },
\end{align*}
where the last inequality is from the inequality $1+x \ge e^{2x/3}$ for $0 < x < 1$.  

\end{proof}

\begin{rem} (Random symmetric oblivious corruption) 
\label{rem:symmetric oblivious corruption}
Lemma \ref{lem:main_lemma_1} applies to the more restricted corruption model, the random symmetric oblivious noise with corruption probability $p < 1$. Recall that in this case, the corruption noise $\epsilon _j$ is a symmetric random variable with probability $p < 1$ and 0 otherwise. In addition, the noise $\epsilon _j$ is also independent with $\ab_j$ and $\xb$, where $y_j = \inner{\xb, \ab_j} + \epsilon _j$.
One can easily check that the same argument applies except we have a sharper inequality in $(ii)$ in the proof of Lemma \ref{lem:main_lemma_1}. 

For the Massart noise $\epsilon $, since it can access $\xb$ and $\ab_j$, we have to consider the worst case, which happens if all the $\text{sign} (\inner{\ub_k, \ab_k} ) = \text{sign} (\inner{\zb_k, \ab_k} ) =   \text{sign} (\inner{\xb, \ab_k} - \inner{\xb_k, \ab_k}  )$ flips or is different with $\text{sign} (\inner{\ub_k, \ab_k} + \lambda^k \epsilon_k/G )$.  

In contrast, for the random symmetric oblivious noise, suppose that the $k$-th response is selected for corruption. Since $\epsilon _k$ is independent with $\ab_k$ (also $\ub_k$ as well since it only depends on $\ab_0, \ab_1, \dots, \ab_{k-1}$) and $\epsilon_k$ is symmetric with respect to $0$, the probability that the sign of $\epsilon_k$ (which is also the sign of $\epsilon_k/G$ since $G >0$) differs from 
the sign of $\inner{\ub_k, \ab_k}$ is at most $1/2$. Thus, the sign of  $\inner{\ub_k, \ab_k} + \lambda^k \epsilon_k/G$ differs the sign of $\inner{\ub_k, \ab_k}$ with probability at most $1/2$.  Since the probability of $k$-th response being selected for corruption is $p$, the overall probability that 
$\text{sign} (\inner{\ub_k, \ab_k} )  \neq \text{sign} (\inner{\ub_k, \ab_k} + \lambda^k \epsilon_k/G)$ is $p/2$. 
Hence, in this case, the inequality $(ii)$ in the proof of Lemma \ref{lem:main_lemma_1} is refined to 
\begin{align*}
&\stackrel{(ii')}{\le}  \big({3 \over 2}  + \lambda^2\big)
- 2 \lambda^2  \mathbb{E} \bigg[  \inner{\ub_k, \ab_k} \Big \{ (1-\frac{p}{2}) \text{sign} (\inner{\ub_k, \ab_k} ) - \frac{p}{2} \cdot \text{sign} (\inner{\ub_k, \ab_k} ) \Big \}  \Big|  \ub_k \bigg]\mathbbm{1}_{ \{a \le \|\ub_k\|^2 < b\} } \\
&\quad\quad\quad\quad=  \big({3 \over 2}  + \lambda^2\big)
- 2 \lambda^2 (1 - p)  \mathbb{E} \bigg[  | \inner{\ub_k, \ab_k}|  \Big|  \ub_k \bigg]\mathbbm{1}_{ \{a \le \|\ub_k\|^2 < b\} }.
\end{align*}
Now, the rest of the proof of the lemma and the main theorem remain the same, except we have the factor $(1-p)$ instead of $(1-2p)$ in Theorem \ref{thm:main_convergence}. This allows the recovery of the signal for SGD-exp for any corruption probability $p < 1$ under the random symmetric oblivious corruption model, answering a related question in \cite{steinerberger2023quantile} for the streaming setting. 
\end{rem}

\section{SGD-exp analysis on streaming ReLU}\label{sec:proofs-relu}
In this section, we show that our approach SGD-exp can be naturally generalized to the robust ReLU regression problem.

\subsection{Initial reductions}
Directly from SGD-exp iteration \eqref{eq:main_iteration-relu}, we have

$$
\zb_{k+1} = \zb_{k} - G \lambda^{-k} \text{sign} ( \sigma(\inner{\xb, \ab_k}) - \sigma( \inner{\xb_k, \ab_k})  + \epsilon _k) \mathbbm{1}_{ \{\inner{\xb_k, \ab_k} \ge 0\} } \ab_k
$$
where $\zb_k := \xb-\xb_k$ is the $k$-th residual vector. 
Similarly to the linear case, we get 
\begin{align}
\label{eq:main_iteration3}
\ub_{k+1} = \lambda \Big \{ \ub_{k} - \text{sign} \left( \sigma(\inner{\xb, \ab_k}) - \sigma( \inner{\xb_k, \ab_k})  + \epsilon_k/G \right) \mathbbm{1}_{ \{\inner{\xb_k, \ab_k} \ge 0\} } \ab_k \Big \} \quad \text{ for } \ub_k := \lambda^k \zb_k/G.
\end{align}
Then, we take the squared Euclidean norm on both sides above to get for $\hat Y_{k+1} = \| \ub_{k+1} \|^2 $ that
\begin{equation}\label{yk-relu}
\hat Y_{k+1} 
= \lambda^2 \Big \{ \| \ub_{k} \|^2 - 2 \inner{\ub_k, \ab_k} \text{sign} \left( \sigma(\inner{\xb, \ab_k}) - \inner{\xb_k, \ab_k}  + \epsilon_k/G \right) \mathbbm{1}_{ \{\inner{\xb_k, \ab_k} \ge 0\} }  + \mathbbm{1}_{ \{\inner{\xb_k, \ab_k} \ge 0\} } \Big \},
\end{equation}
where we have used the fact that  $\|\ab_k\| = 1$ and $ \sigma( \inner{\xb_k, \ab_k}) = \inner{\xb_k, \ab_k}$ on the event  $\inner{\xb_k, \ab_k} \ge 0$. 
The following lemma shows how the expression further simplifies on the event $\epsilon _k = 0$ (no corruption event)
\begin{lem}
\label{lem:noiseless_sign_reduction}
    As long as $\PP(\epsilon _k \ne 0) \le p$, we have
    \[
        \text{sign} \left( \sigma(\inner{\xb, \ab_k}) - \inner{\xb_k, \ab_k}  + \epsilon_k/G  \right)\mathbbm{1}_{ \{\inner{\xb_k, \ab_k} \ge 0\} } 
        \ge ((1 - p) \text{sign}( \inner{\ub_k, \ab_k}) - p)\mathbbm{1}_{ \{\inner{\xb_k, \ab_k} \ge 0\} }.
    \]
\end{lem}
\begin{proof} First, we note that conditioned on the event that $\epsilon _k = 0$, \begin{align*}
    \text{sign} \left( \sigma(\inner{\xb, \ab_k}) - \inner{\xb_k, \ab_k} \right)\mathbbm{1}_{ \{\inner{\xb_k, \ab_k} \ge 0\} } &= \text{sign} \left( \inner{\xb, \ab_k} - \inner{\xb_k, \ab_k} \right)\mathbbm{1}_{ \{\inner{\xb_k, \ab_k} \ge 0\} } \\ &= \text{sign} \left( \inner{\ub_k, \ab_k} \right)\mathbbm{1}_{ \{\inner{\xb_k, \ab_k} \ge 0\} }.
\end{align*}
Indeed, if $\inner{\xb, \ab_k}  < 0$ then both expressions under the sign are negative, and if $\inner{\xb, \ab_k}  \ge 0$ then both expressions are the same. We have the statement of the lemma by the law of total probability. 
\end{proof}

\subsection{SGD drift estimates and proof of convergence theorem for ReLU regression}

We have the following result similar to Lemma~\ref{lem:main_lemma_1}:
\begin{lem} 
\label{lem:main_lemma_ReLU_1}
Let the random process $\hat Y_k$ be defined as per \eqref{yk-relu}, where $\{ {\ab}_j \}_{j=1}^{n}$ are i.i.d. random vectors in $\R^d$ satisfying Assumptions \ref{assumption:measurement_model} and \ref{assumption:measurement_model_ReLU} with some constant $\tilde{C} > 0$ and the corruption noise $\epsilon _j$ follows the Massart model \ref{m-noise}. Let $\mathcal{F}_k$ be the $\sigma$-algebra generated by \\
$\{{\ab}_0, \epsilon _0\}, \{{\ab}_1, \epsilon _1\}, \dots, \{{\ab}_{k-1}, \epsilon _{k-1}\}$. Then for any step size decay parameter $\lambda > 0$ that satisfies
\begin{equation}\label{lam-bound-ReLU}
1 < \lambda^2 \le  1 +   \widetilde{C}^2{(1-2p)^2 \over 49d} <  \frac{50}{49},
\end{equation}
for $a = \frac{1}{2(\lambda^2 -1)}, \;
b = \frac{3}{2(\lambda^2 - 1)}$ and
$$
\eta = c^*\sqrt{\lambda^2-1};  \quad  c^* = {1 \over 8 \lambda^2}  \left[ {  \lambda^2 (1-2p) \widetilde{C}  \over \sqrt{2} \sqrt{d}} - \sqrt{\lambda^2-1}({3 \over 2}  + {\lambda^2 \over 2}  ) \right]
$$
we have
    \begin{align*}
&\mathbb{E} [e^{\eta (Y_{k+1} - Y_k )} \mathbbm{1}_{ \{a \le Y_k < b\} } | \mathcal{F}_k] \le \left( 1 - {\widetilde{C}^2(1-2p)^2  \over 100 d } \right) 
\end{align*}

\end{lem}

\begin{proof}
\small
\begin{align*}
&\mathbb{E} [ (Y_{k+1} - Y_k) \mathbbm{1}_{ \{a \le Y_k < b\} } | \mathcal{F}_k ] \\ 
& = \mathbb{E} \Big [ \left(\|\ub_{k+1}\|^2 - \|\ub_k\|^2\right) \mathbbm{1}_{ \{a \le Y_k < b\} } \Big| \mathcal{F}_k \Big ] \\
& = \mathbb{E} \Big [ \left( (\lambda^2-1)\|\ub_k\|^2 - \lambda^2 \left(2 \inner{\ub_k, \ab_k} \text{sign} \left( \sigma(\inner{\xb, \ab_k}) - \inner{\xb_k, \ab_k}  + \epsilon_k/G \right)  - 1 \right) \mathbbm{1}_{ \{\inner{\xb_k, \ab_k} \ge 0\} } \right) \mathbbm{1}_{ \{a \le \|\ub_k\|^2 < b\} } \Big|  \mathcal{F}_k  \Big ] \\
&  \stackrel{(i)}{\le}  \left({3 \over 2}  + {\lambda^2 \mathbb{E}_{A} \bigg[  \mathbbm{1}_{ \{\inner{\xb_k, \ab_k} \ge 0\} } \Big|  \ub_k \bigg]} \right) \mathbbm{1}_{ \{a \le \|\ub_k\|^2 < b\} }  - 2 \lambda^2  \mathbb{E}_{A} \bigg[ (1 - 2p) | \inner{\ub_k, \ab_k}| \mathbbm{1}_{ \{\inner{\xb_k, \ab_k} \ge 0\}  } \Big|  \ub_k \bigg]\mathbbm{1}_{ \{a \le \|\ub_k\|^2 < b\}} \\
&\stackrel{(ii)} {=} \left({3 \over 2}  + {\lambda^2 \over 2} \right) -  \lambda^2 (1 - 2p) \mathbbm{1}_{ \{a \le \|\ub_k\|^2 < b\} } \mathbb{E}_{A} \bigg[  | \inner{\ub_k, \ab_k}|  \Big|  \ub_k \bigg]\\
&\le  \left({3 \over 2}  + {\lambda^2 \over 2} \right) - {  \widetilde{C} (1-2p) \lambda^2 \over  \sqrt{d}\sqrt{2(\lambda^2-1)}} \mathbbm{1}_{ \{a \le Y_k < b\} }.
\end{align*}
\normalsize

Here, the step (i) requires the assumption on the size of $a$ and Lemma \ref{lem:noiseless_sign_reduction} ,  the step (ii) requires Assumption \ref{assumption:measurement_model_ReLU}, and finally we use Assumption~\ref{assumption:measurement_model} in the last step. 
Note the similarity with equation \eqref{first-term}. We proceed in parallel with the proof of Lemma~\ref{lem:main_lemma_1} to get using Taylor expansion
\begin{align*}
\mathbb{E} [e^{c \sqrt{\lambda^2-1} (Y_{k+1} - Y_k )} \mathbbm{1}_{ \{a \le Y_k < b\} } | \mathcal{F}_k] &
 \le 1 - u c  
+ 4 \lambda^4 c^2, \\
& \text{ where } \quad 
u =  {\lambda^2 (1-2p) \widetilde{C}  \over  \sqrt{2}\sqrt{d}} - \sqrt{\lambda^2-1}({3 \over 2}  + {\lambda^2 \over 2}  ).
\end{align*}
From the condition $1 < \lambda^2 < 1 +  \widetilde{C}^2{(1-2p)^2 \over 49d}$  we can take
\begin{equation*}c^{*} = {u \over 8 \lambda^4}  > {(1/\sqrt{2} - 3/14 - 1/14) \over 8 \lambda^2}   {   \widetilde{C} (1-2p)  \over \sqrt{d}} > {1 \over 20} \bigg[  { \widetilde{C} (1-2p) \over \sqrt{d}}   \bigg] \end{equation*}
This gives
\begin{align*}
&\mathbb{E} [e^{c^*\sqrt{\lambda^2 -1} (Y_{k+1} - Y_k )} \mathbbm{1}_{ \{a \le Y_k < b\} } |  \mathcal{F}_k] \le 1 - 4 \lambda^4 (c^*)^2 \le  1 - {\widetilde{C}^2(1-2p)^2  \over 100 d }.
\end{align*}
\end{proof}

With the proof directly following the proof of Lemma \ref{lem:main_lemma_2}, we also get:
\begin{lem} 
\label{lem:main_lemma_ReLU_2} Let the random process $Y_k$ be defined as per \eqref{yk-def}, where $\{ {\ab}_j \}_{j=1}^{n}$ are i.i.d. random vectors in $\R^d$ satisfying Assumption 1 with some constant $\tilde{C} > 0$ and the corruption noise $\epsilon _j$ is non-zero with probability $p < 1/2$. Let $\mathcal{F}_k$ be the $\sigma$-algebra generated by \\
$\{{\ab}_0, \epsilon _0\}, \{{\ab}_1, \epsilon _1\}, \dots, \{{\ab}_{k-1}, \epsilon _{k-1}\}$. Then, for any step size decay parameter $\lambda > 0$ that satisfies 
\begin{equation*}
1 < \lambda^2 \le  1 +   \widetilde{C}^2{(1-2p)^2 \over 49d} <  \frac{50}{49},
\end{equation*}
for 
$$
a = \frac{1}{2(\lambda^2 -1)}; \quad \eta = c^*\sqrt{\lambda^2-1};  \quad  c^* = {1 \over 8 \lambda^2}  \left[ {  \lambda^2 (1-2p) \widetilde{C}  \over \sqrt{2} \sqrt{d}} - \sqrt{\lambda^2-1}({3 \over 2}  + {\lambda^2 \over 2}  ) \right]
$$
we have
\begin{align*}
&\mathbb{E} [e^{\eta (Y_{k+1} - a )} \mathbbm{1}_{ \{ Y_k < a\}} |  \mathcal{F}_k] 
\le \exp\Big \{ {\widetilde{C} (1-2p) \over 6 \sqrt{d}}  \Big \}  .
\end{align*}

\end{lem}

\begin{proof}[Proof of Theorem \ref{thm:main_convergence_relu}] This is a direct application of the estimate from Corollary~\ref{cor:stopping_time_bound} with the constants $K, D$ and $\rho$ as obtained by Lemma~\ref{lem:main_lemma_ReLU_1} and Lemma~\ref{lem:main_lemma_ReLU_2}.
\end{proof}

\begin{rem} (Random symmetric oblivious corruption for ReLU regression) 
\label{rem:symmetric oblivious corruption_ReLU}
For the random symmetric oblivious noise, suppose that the $k$-th response is selected for corruption. The probability that the sign of $\epsilon _k$ differs from
the sign of $\sigma(\inner{\xb, \ab_k}) - \inner{\xb_k, \ab_k}$ is at most $1/2$. Thus, the sign of  $\sigma(\inner{\xb, \ab_k}) - \inner{\xb_k, \ab_k} + \epsilon_k/G$ differs from the sign of $\sigma(\inner{\xb, \ab_k}) - \inner{\xb_k, \ab_k}$ with probability at most $1/2$ since $G > 0$.  Since the probability of $k$-th response being selected for corruption is $p$, the overall probability that 
$\text{sign} (\sigma(\inner{\xb, \ab_k}) - \inner{\xb_k, \ab_k} )  \neq \text{sign} (\sigma(\inner{\xb, \ab_k}) - \inner{\xb_k, \ab_k} +  \epsilon_k/G)$ is $p/2$. 
Hence, in this case, the inequality $(i)$ in the proof of Lemma \ref{lem:main_lemma_ReLU_1} is refined to 
\small
\begin{align*}
 \mathbb{E} &\Big [ \left( (\lambda^2-1)\|\ub_k\|^2 - \lambda^2 \left(2 \inner{\ub_k, \ab_k} \text{sign} \left( \sigma(\inner{\xb, \ab_k}) - \inner{\xb_k, \ab_k}  + \epsilon_k/G \right)  - 1 \right) \mathbbm{1}_{ \{\inner{\xb_k, \ab_k} \ge 0\} } \right) \mathbbm{1}_{ \{a \le \|\ub_k\|^2 < b\} } \Big|  \mathcal{F}_k  \Big ] \\
&  \stackrel{(i)}{\le}  \left({3 \over 2}  + {\lambda^2 \mathbb{E}_{A} \bigg[  \mathbbm{1}_{ \{\inner{\xb_k, \ab_k} \ge 0\} } \Big|  \ub_k \bigg]} \right) \mathbbm{1}_{ \{a \le \|\ub_k\|^2 < b\} }  \\
&\quad\quad\quad\quad\quad\quad\quad\quad - 2 \lambda^2  \mathbb{E}_{A} \bigg[ (1 - p) | \inner{\ub_k, \ab_k}| \mathbbm{1}_{ \{\inner{\xb_k, \ab_k} \ge 0\}  } \Big|  \ub_k \bigg]\mathbbm{1}_{ \{a \le \|\ub_k\|^2 < b\}}.
\end{align*}
\normalsize
Now, the rest of the proof of Lemma \ref{lem:main_lemma_ReLU_1} and Theorem~\ref{thm:main_convergence_relu} remain the same, except we have the factor $(1-p)$ instead of $(1-2p)$ in Theorem \ref{thm:main_convergence}. This allows the recovery of the signal for SGD-exp for any corruption probability $p < 1$ under the random symmetric oblivious corruption model for ReLU regression, same as in the linear case. 
\end{rem}

\begin{rem}
One advantage of SGD-exp is its rapid convergence rate for Massart noise in the streaming setting. Since our work has appeared, Massart noise in the streaming setting has garnered attention from other areas, such as bandit optimization in streaming settings \cite{diakonikolas2024online}. However, this aggressive step-size decay in SGD-exp may not be optimal under Gaussian noise, where a slower decay could eventually lead to better convergence error in the long run.  Nevertheless, in practical deep learning, exponential step size schedules are still widely used and empirically successful. To achieve the best of both worlds, we aim to develop more sophisticated step-size scheduling schemes, such as hybrid strategies, that could handle both types of noise. This is a very recent direction of research, which has so far been explored only under non-adversarial noise settings, for example, in \cite{wang2023convergence}.
\end{rem}

\section{Experiments}\label{sec:experiments}
In this section, we report on numerical experiments designed to validate our theoretical findings and demonstrate the performance of SGD-exp.

    \subsection{Experiments on random systems}
    \label{section:experiments_synthetic_datasets}
In each trial of   experiments in Section \ref{section:experiments_synthetic_datasets}, we generate the measurement vectors $\ab_k$ and $\xb$ in $\R^d$ randomly according to the $d$-dimensional standard Gaussian random distribution. 
The effectiveness of the tested methods is evaluated by measuring the relative L2 error, ${\|\xb - \xb_k\|_2 \over \|\xb\|_2}$ in the number of iterations or data point access. We will further refer to it as relative error. 
    \subsubsection{SGD-exp vs SGD-root}

    \begin{figure}[h] 
	
	   \centering
        \includegraphics[width=0.45 \textwidth]{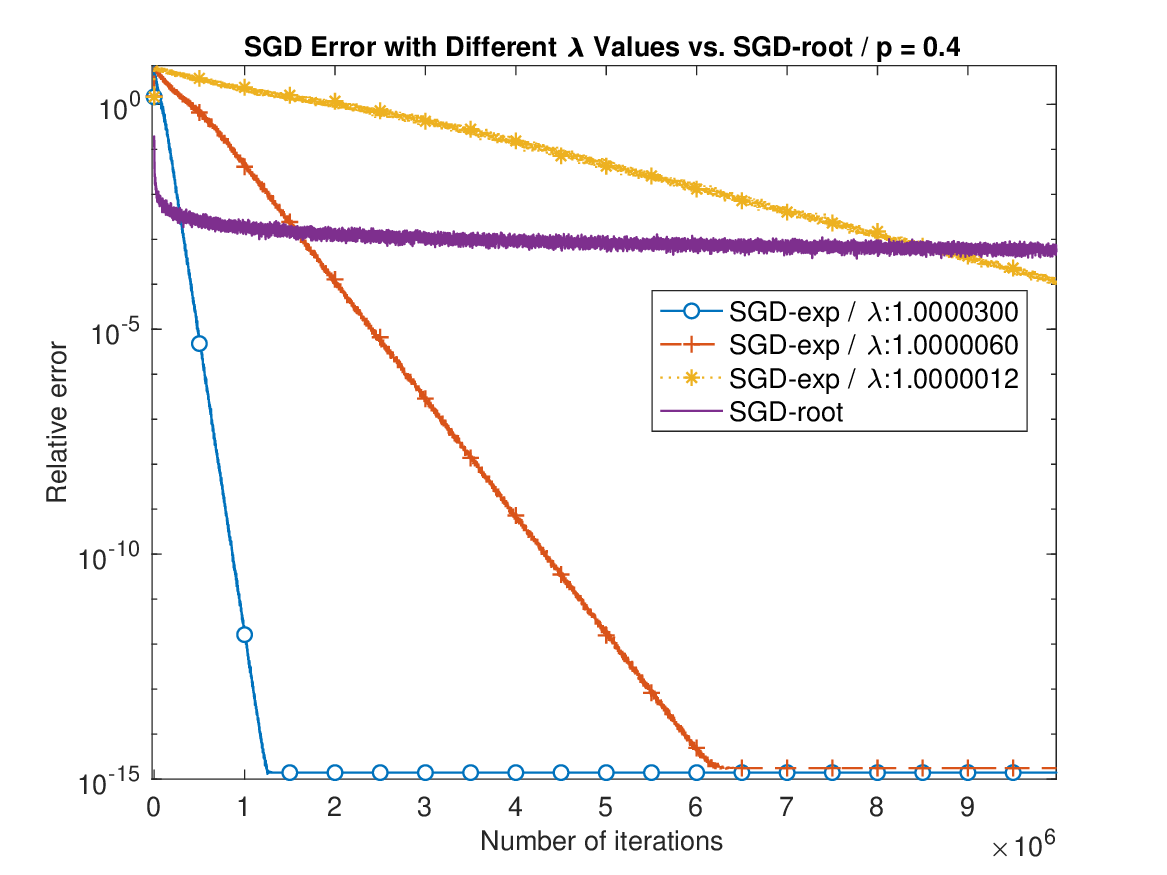}
		\caption[] {Linear regression with SGD-exp and SGD-root on streaming Gaussian data with corruption probability $p = 0.4$ (sign-flip corruptions).  For SGD-exp, step size scales $\lambda = 1.00003$, $\lambda = 1.000006$, $\lambda = 1.0000012$. Larger $\lambda$ results in faster convergence but even very small $\lambda$ are more efficient than SGD-root in the long run. The plots are averaged over $20$ runs.
}
        \label{fig:SGD-streaming}
    \end{figure}
    The semilog plots in Figure \ref{fig:SGD-streaming} imply that SGD-exp converges linearly for a smaller step size decay rate $\lambda$, whereas SGD-root only provides a sublinear convergence. Here, the measurement vectors are the $100$-dimensional normalized Gaussian vectors, and the relative errors are averaged over $20$ trials.
  
    \subsubsection{SGD-exp with various values of $p$ and step sizes for sign-flip corruption}
    Figure \ref{fig:SGD-exp_streaming_various_p_step_sizes} displays the relative error plots of SGD-exp for robust linear regression with various choices of step sizes and different corruption probabilities $p$. The measurement vectors are the $100$-dimensional normalized Gaussian vectors. 
    
    The corruption model is the sign corruption, in other words, the $i$-th measurement $y_i$ is replaced with $- y_i$ with probability $p$. This is a typical example of the Massart noise since it requires the knowledge of the sign of $\inner{\xb, \ab_j}$. 
    We have selected this model because, in random symmetric oblivious corruption models that involve adding large random errors, SGD-exp still converges to the true parameter even when  $p > 1/2$ as predicted by our theory and demonstrated by the experiments in the following subsection.

    The plots in Figure \ref{fig:SGD-exp_streaming_various_p_step_sizes} indicate that as the step size decay factor $\lambda$ gets close to $1$, SGD-exp converges for higher values of the corruption probability $p$ at the expense of the convergence rates. The right plot confirms our theory that SGD-exp still converges when the corruption probability is close to $0.5$ for the Massart noise type corruption. 

    \begin{figure}[h] 
	
			\centering
     \includegraphics[width=0.45 \textwidth]{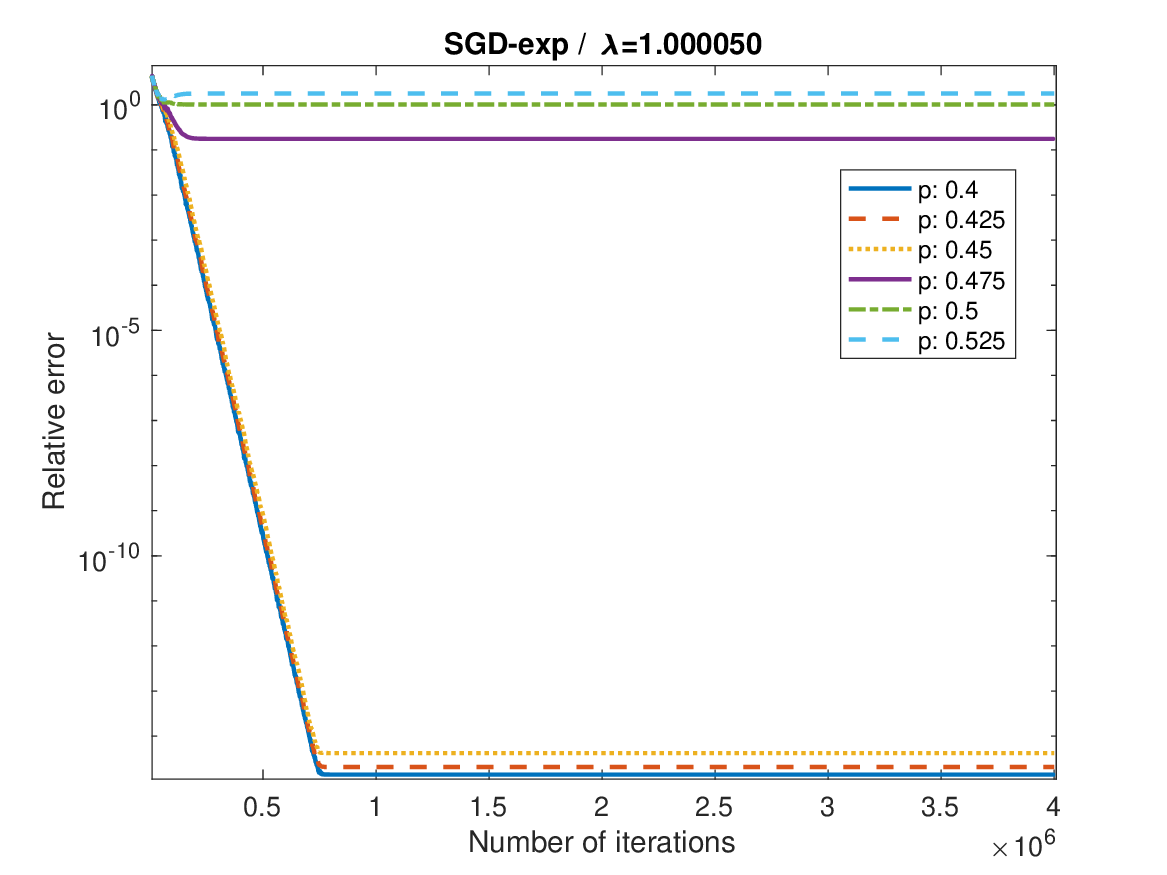}
     \includegraphics[width=0.45 \textwidth]{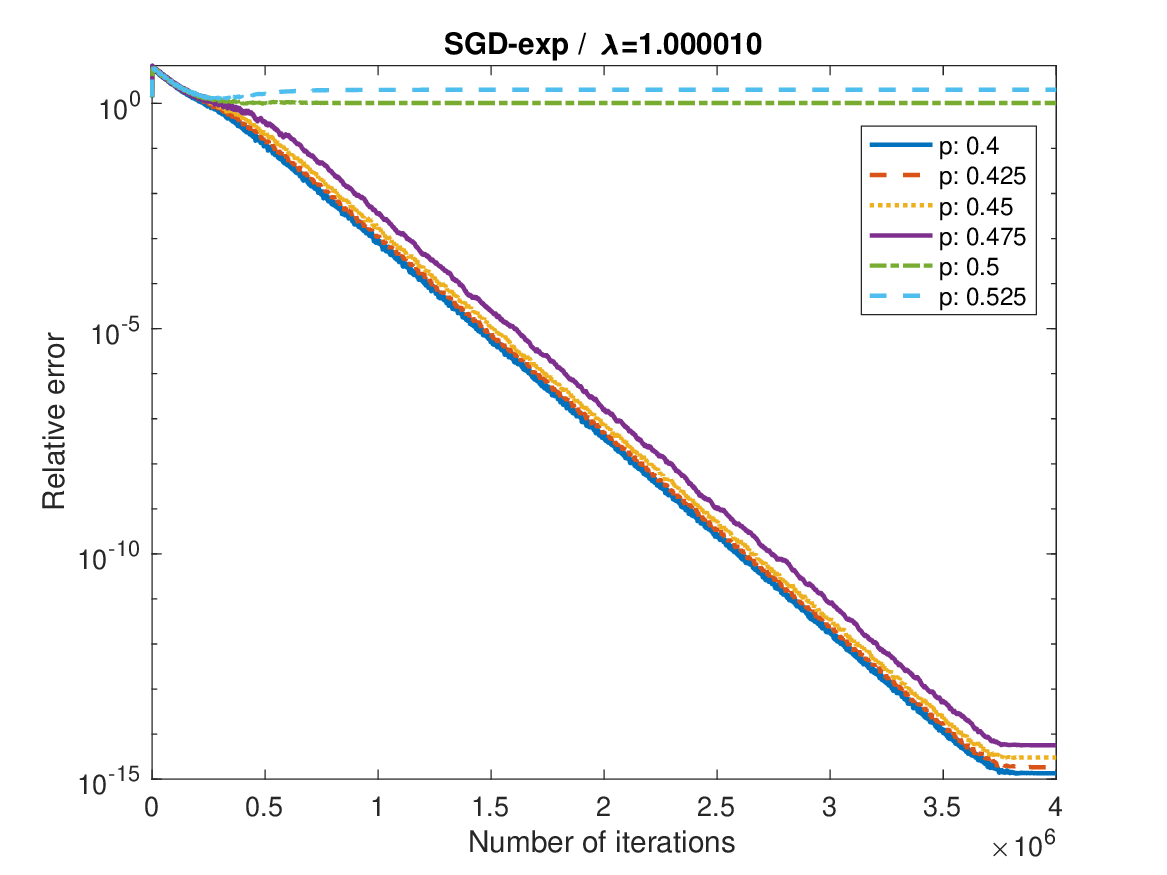}

	\caption[] {Linear regression with SGD-exp on streaming Gaussian data with different values of the corruption probability $p$ (sign-flip corruptions). With step size $\lambda = 1.00005$ (left)

     the error of SGD-exp converges linearly for  $p \le 0.45$. With more conservative step size  $\lambda = 1.00001$ (right) the error of SGD-exp converges linearly for  $p \le 0.475$. The plots are averaged over $10$ runs.
	}

    \label{fig:SGD-exp_streaming_various_p_step_sizes}
    \end{figure}

    \subsubsection{SGD-exp with various values of $p$ with recommended step size for the symmetric oblivious corruption and the sign-flip corruption}

    Figure \ref{fig:SGD-exp_Gaussian_noise} illustrates the error plots of SGD-exp for robust linear regression with several corruption probability $p$ and associated step size parameter $\lambda$. Here, the measurement vectors are the $50$-dimensional rescaled standard normal Gaussian vectors. 
    The response $y_i$ is corrupted by the random symmetric oblivious additive Gaussian noise with variance $30$. Following the guideline about the dependence of the parameter $\lambda$ on $T, d$, and $p$ as stated in Theorem \ref{thm:main_convergence}, we set $\lambda = \sqrt{1 + 2{(1-p)^2 \over {d \log^2 T}} }$ for various values of the corruption probability $p$.
    
    The error plot on the left in Figure \ref{fig:SGD-exp_Gaussian_noise} supports our theory that SGD-exp still converges to the solution for any corruption probability $p < 1$ under the random symmetric oblivious corruption model. Similarly, the relative error plot on the right in Figure \ref{fig:SGD-exp_Gaussian_noise} validates our theory that SGD-exp still converges to the solution for any corruption probability $p < 1/2$ under the Massart corruption noise. 

 \begin{figure}[h] 

	\centering

    \includegraphics[width=0.45\textwidth] {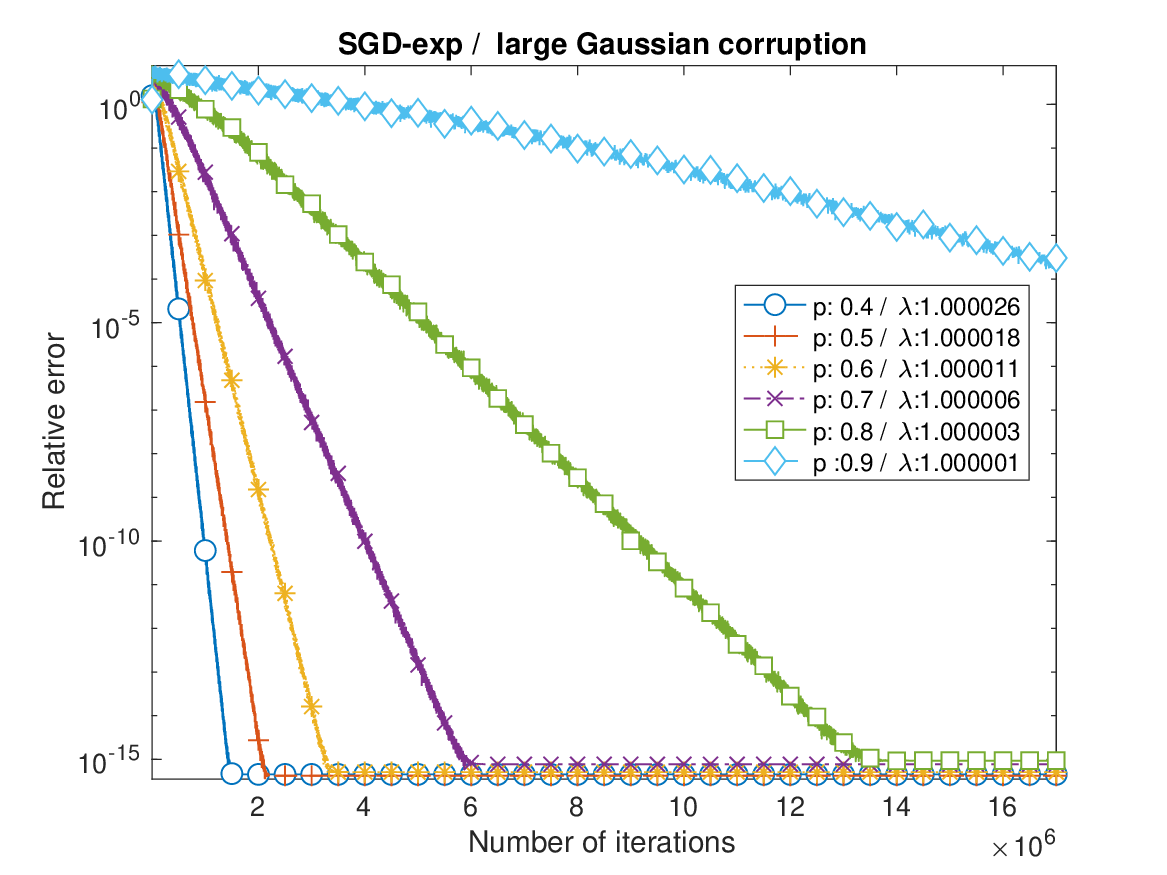}
     \includegraphics[width=0.45\textwidth] {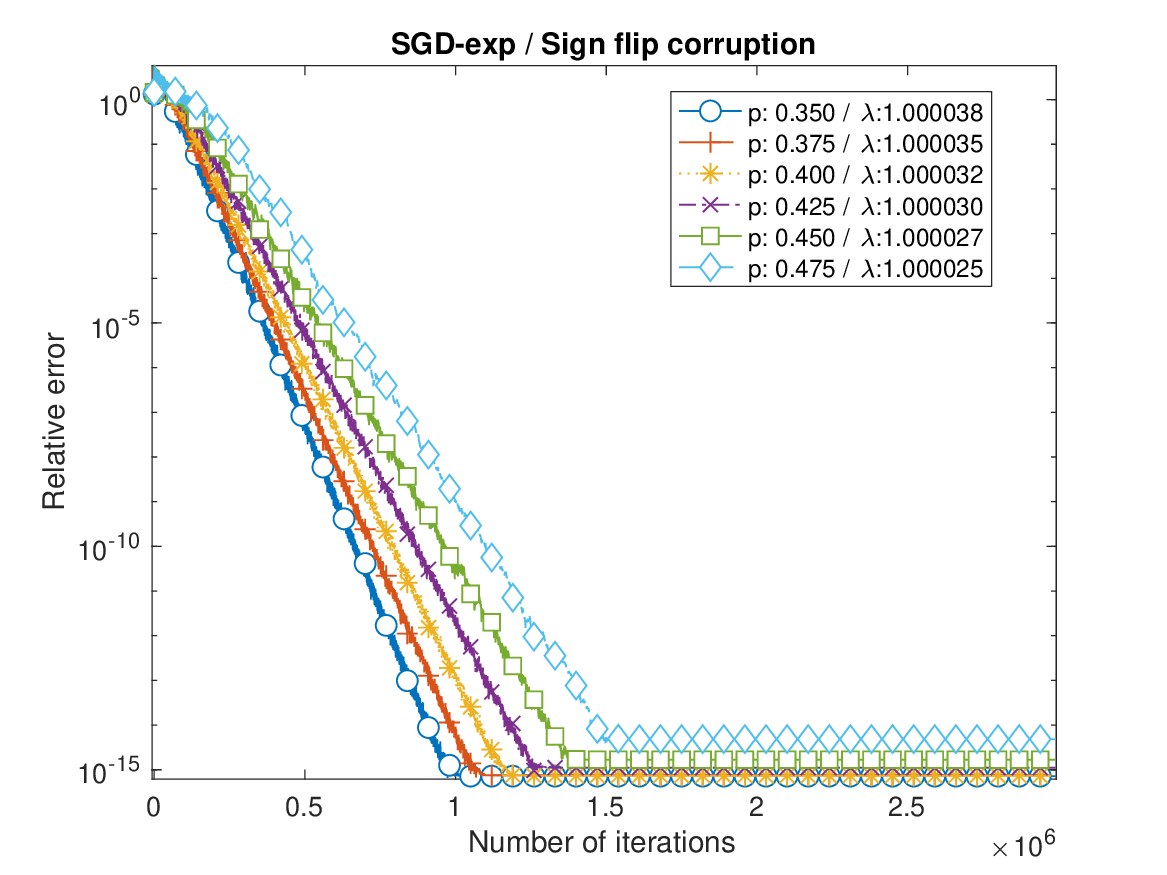}

		\caption[] {Linear regression with SGD-exp on streaming Gaussian data with step-size $\lambda$ recommended by Theorem~\ref{thm:main_convergence}.  The error of SGD-exp converges approximately linearly for $p \le 0.9$ for the symmetric random oblivious corruption model (left) for $p \le 0.475$ for the sign-flip error (right). The plots are averaged over $10$ runs.
		}

        \label{fig:SGD-exp_Gaussian_noise}
    \end{figure}

    \subsubsection{SGD-exp for robust ReLU regression}

    In this subsection, we illustrate the effectiveness of SGD-exp for robust ReLU regression. Our findings suggest that SGD-exp outperforms GLM-Tron \cite{kakade2011efficient,karmakar2022provable,wu2023finite}, the popular ReLU regression method. GLM-Tron is essentially a (stochastic) iterative method for ReLU regression based on the $\ell_2$ loss. Although works such as \cite{wu2023finite} have provided some robustness property of GLM-Tron, our findings indicate that it is not robust with respect to the Massart noise, which is our corruption model. In particular, as shown in Figure \ref{fig:Comparision_GLM_Tron_SGD-exp}, SGD-exp provides (nearly) linear convergence for robust ReLU regression under Massart noise, while stochastic GLM-Tron with constant, polynomial and exponential step size decays (which we denote for brevity as GLM-Tron-constant, GLM-Tron-root and GLM-Tron-exp and define precisely in the caption of Figure~\ref{fig:Comparision_GLM_Tron_SGD-exp}) all fail to converge. We note that on the same data without corruptions (when $p = 0$) the same GLM-Tron-exp and GLM-Tron-root methods converge successfully, achieving relative error of $10^{-15}$ and $10^{-2}$ respectively in $2\cdot 10^{5}$ iteration steps.
    
  \begin{figure}[h] 
	\centering
     \includegraphics[width=0.45\textwidth] {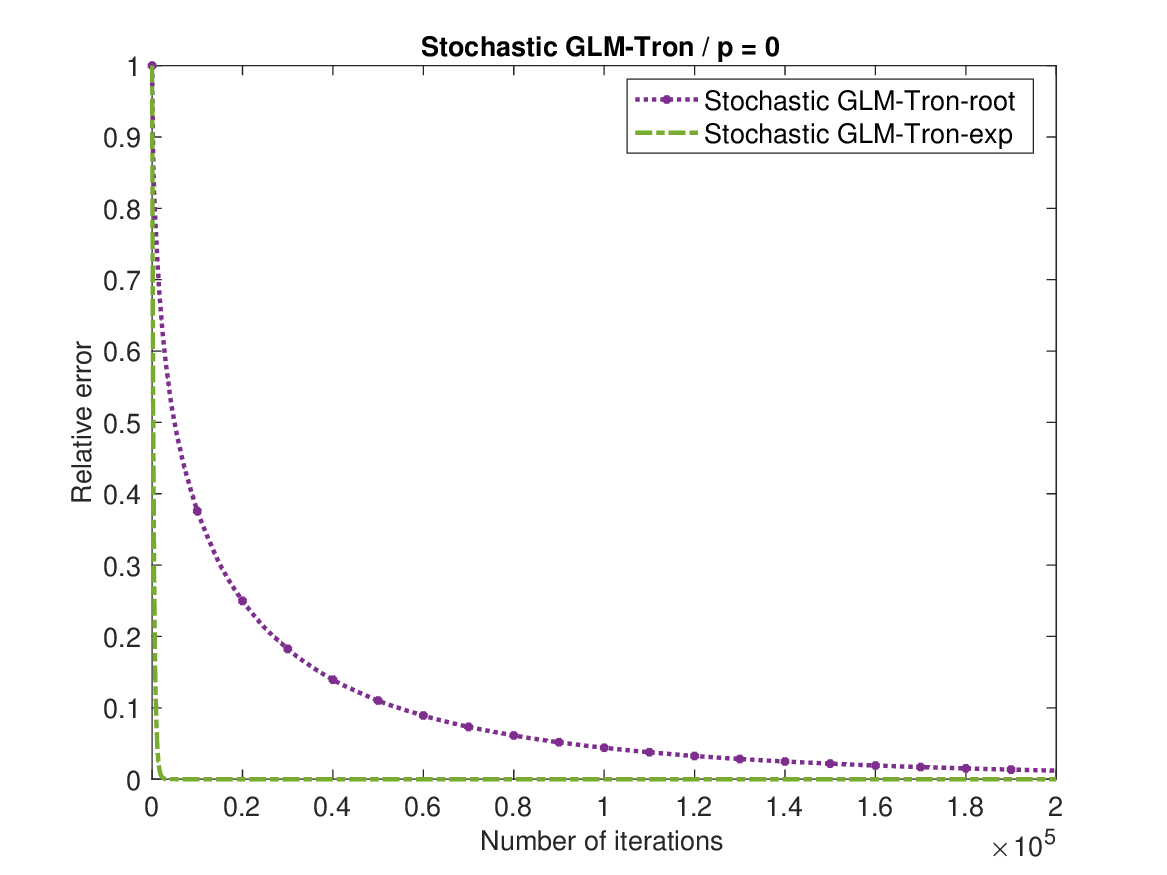}
     \includegraphics[width=0.45\textwidth] {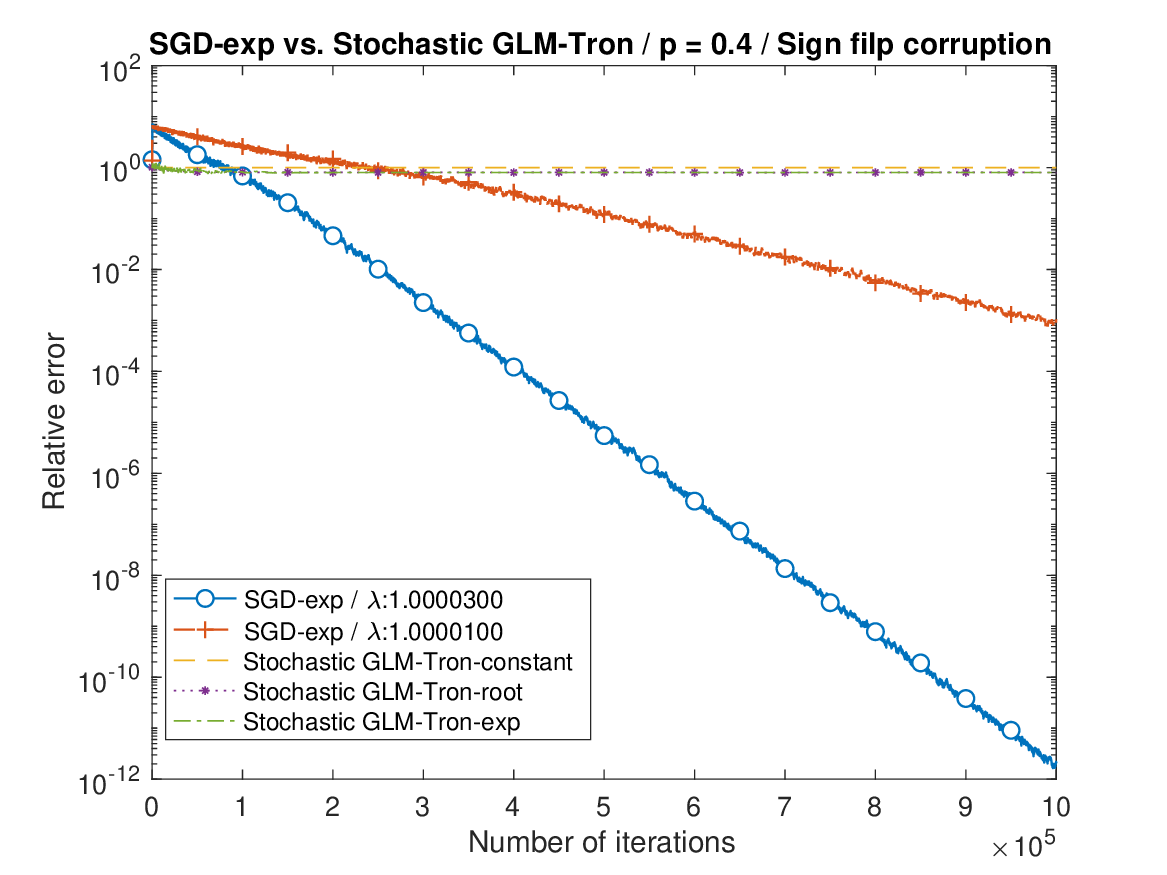}

		\caption[] {
        Stochastic GLM-Tron in the streaming setting with no corruption ($p=0$) offers recovery of the true parameter for ReLU regression with both square root decaying and exponentially decaying step size scheduling (left). The number of samples in the data set is denoted by $m$. Here, GLM-Tron-exp employs exponentially decaying step size $1.00003^{-k}/m$, GLM-Tron-const has step size $1/m$ and GLM-Tron-root has step size $k^{-1/2}/m$.  
        However, with sign corruption with $p = 0.4$  (right), stochastic GLM-Tron struggles with various step size choices (the error curves are similar). In contrast, the error of ReLU-SGD-exp converges linearly in robust ReLU regression with sign corruption. The plots are averages over $10$ runs. 
  } 
        \label{fig:Comparision_GLM_Tron_SGD-exp}
    \end{figure}
      \begin{figure}[h] 

	\centering
     \includegraphics[width=0.45\textwidth] {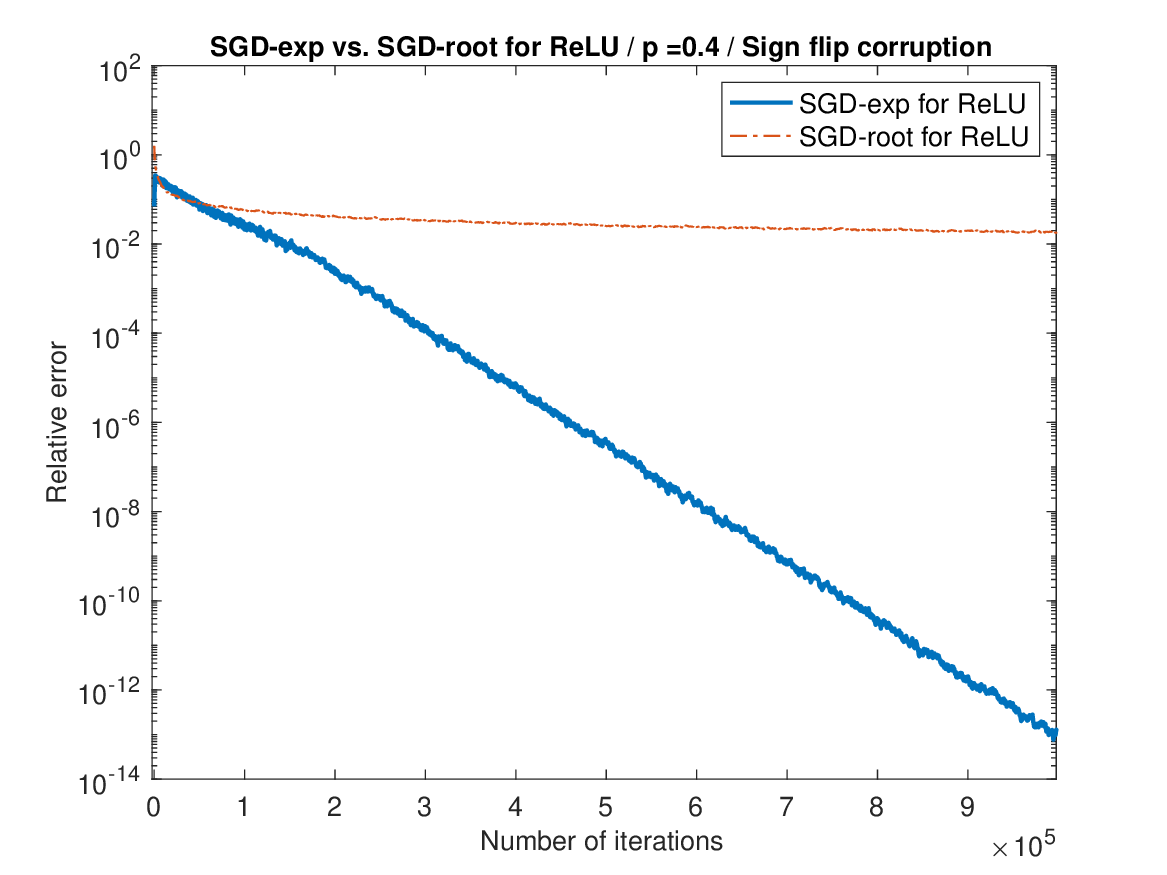}
     \includegraphics[width=0.45\textwidth] {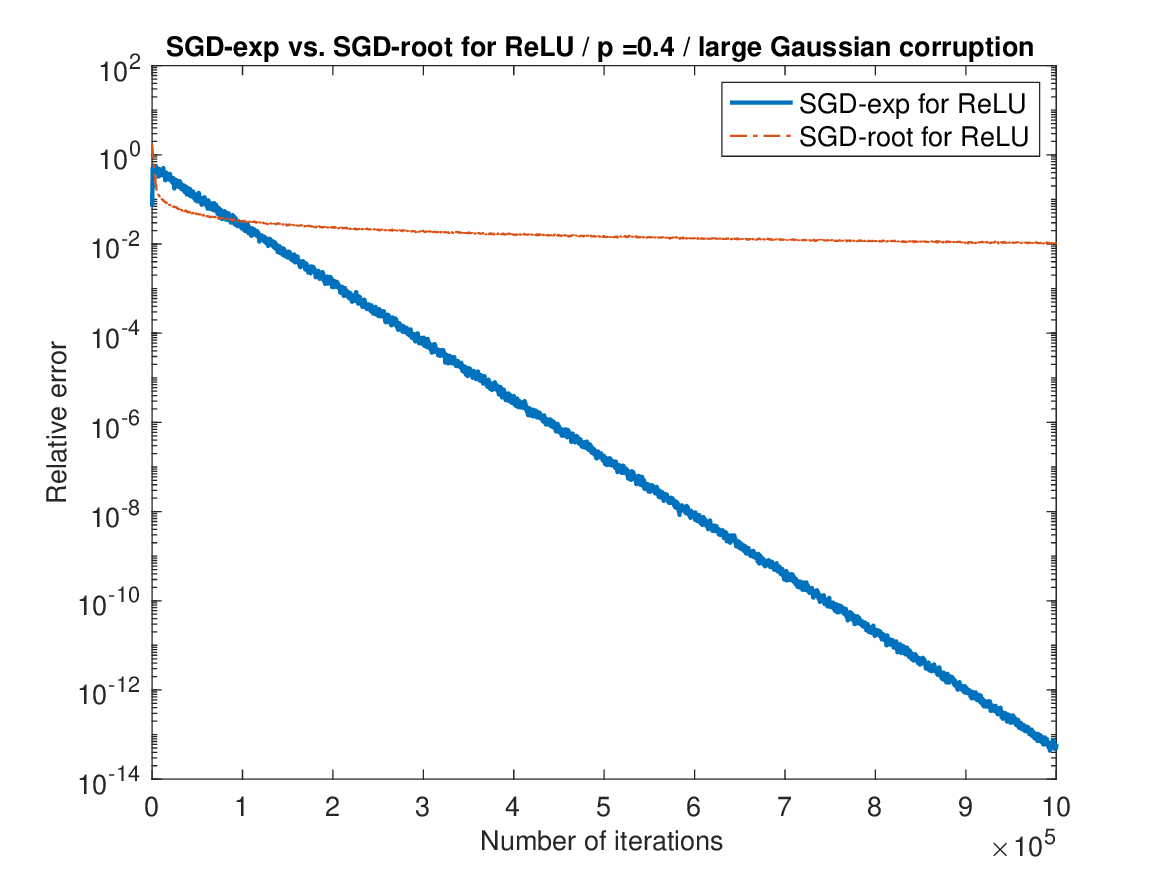}

		\caption[] {ReLU regression with SGD-exp and SGD-root on streaming Gaussian data with corruption probability $p = 0.4$: sign-flip corruptions (left) and large Gaussian corruptions (right). For SGD-exp, step size scales $\lambda = 1.00003$.

        In both cases, SGD-exp offers linear convergence whereas SGD-root fails to converge to the true parameter.  
		} 
        \label{fig:Comparision_SGD-exp_SGD-root}
    \end{figure}
    One might also try different step size scheduling such as the square-root decay step size instead of the exponentially decaying step size in the iteration equation \eqref{eq:main_iteration-relu} of SGD-exp for robust ReLU regression. In fact, the iteration equation \eqref{eq:main_iteration-relu} with square-root decay step size can be viewed as an extension of SGD-root \cite{pesme2020online} to ReLU regression. However, our findings suggest that this extension of SGD-root does not converge to the true parameter as shown in Figure \ref{fig:Comparision_SGD-exp_SGD-root} for sign-flip/large Gaussian corruption, whereas SGD-exp provides linear convergence for both types of corruption. This indicates that the exponential decaying step size in SGD-exp is a correct choice for robust ReLU regression under the Massart noise model.

    \subsection{SGD-exp on real datasets}
    In this section, we demonstrate the effectiveness of SGD-exp for linear and ReLU regression using real datasets with corrupted data. We apply our methods on two datasets, Red Wind Quality Data and Lending Club Loan Data and record the average loss values over $10$ trials. 
    
    \subsubsection{Red Wine Quality Data}
    We use Red Wine Quality, a popular dataset for linear regression. The dataset consists of 1599 samples with physicochemical covariates and sensory response variable. Among the covariate variables, we only use 10 numerical ones for simplicity. The $10$ features are  \verb|fixedAcidity|, \verb|volatileAcidity|, \verb|citricAcid|, \verb|residualSugar|, \verb|chlorides|, \verb|freeSulfurDioxide|, \verb|density|, \verb|pH|, \verb|sulphates|, \verb|alcohol|. We centralize/normalize these features and corrupt the response variable by adding large random Gaussian noise drawn from $[-300, 300]$ with probability $p = 0.2$. The parameter $\lambda = 1.006$. A similar corruption model with real dataset has been used in \cite{haddock2022quantile}. 

To simulate the i.i.d. samples, we randomly select one of the data point from the dataset at each iteration of SGD-exp.
    Because typically there is no such reasonable ground truth vector $\xb$ for the real datasets, we measure the performance of methods using the $\ell_2$ loss function associated uncorrupted dataset instead. 
    In Figure \ref{fig:SGD-exp_real_dataset}, we record the $\ell_2$ loss of the uncorrupted dataset, ${1 \over m} \|A\xb - \tilde{{\bf y}} \|^2$, where $\tilde{{\bf y}}$ is the corresponding uncorrupted response vector. 
    
    The loss value after one pass of the corrupted dataset using SGD-exp for linear regression is about \verb|0.463|. For smaller scale of noise, drawn from $[-3,3]$, a similar plot to that in Figure \ref{fig:SGD-exp_real_dataset} is obtained, which is omitted here, where we have obtained the loss value \verb|0.450|.  These values are close to the optimal value \verb|0.4220| which can be obtained by running the conventional linear regression on the uncorrupted dataset. If we run the linear regression on the corrupted dataset, the loss value is over \verb|32.11|, much higher than the one associated with the uncorrupted dataset. 

    As for robust ReLU regression, we obtain a similar plot for SGD-exp for ReLU regression, whereas GLM-Tron suffers from the corruption as shown in Figure \ref{fig:Comparison_real_dataset_relu}.
    
 \begin{figure} [h]
      \centering

     \begin{minipage}[b]{0.44\textwidth}
        \includegraphics[width=1 \textwidth]{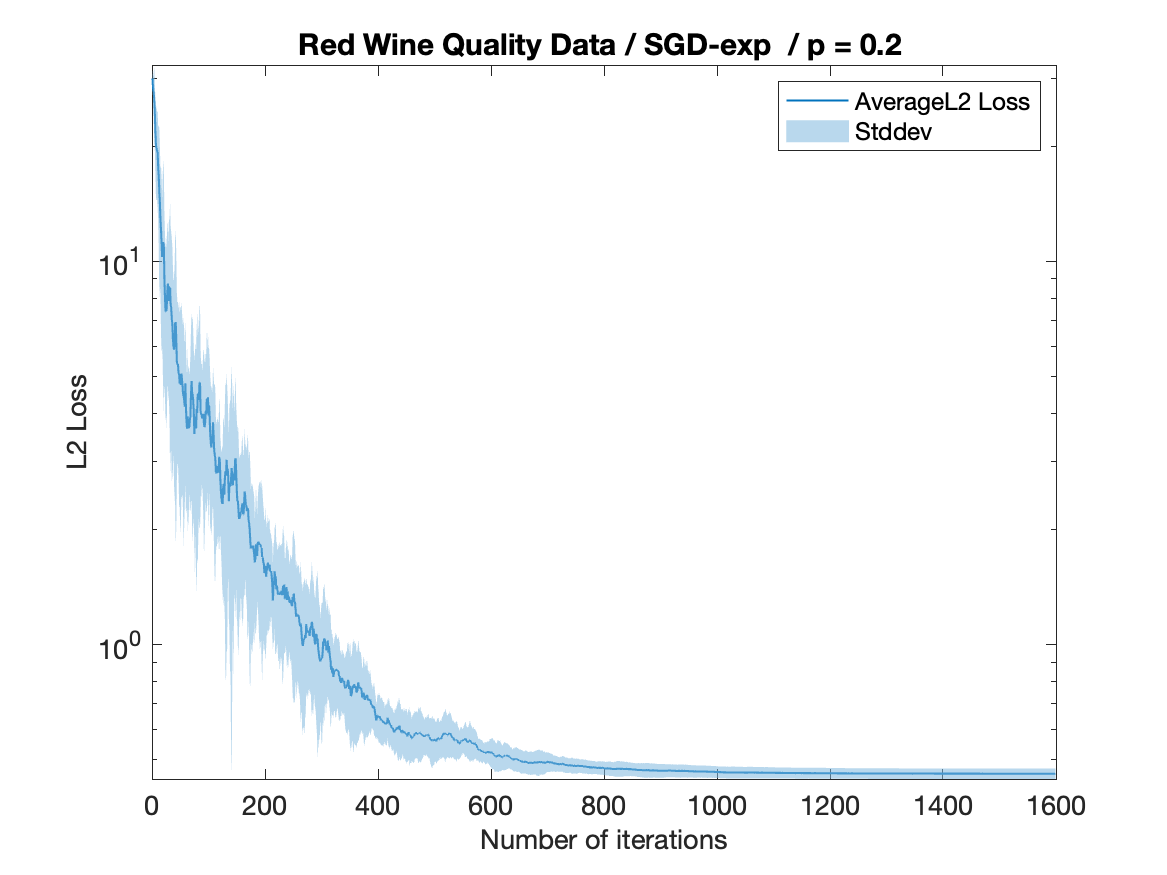}
      \end{minipage}
    \hspace{1cm}
      \begin{minipage}[b]{0.44\textwidth}
        \includegraphics[width=1 \textwidth]{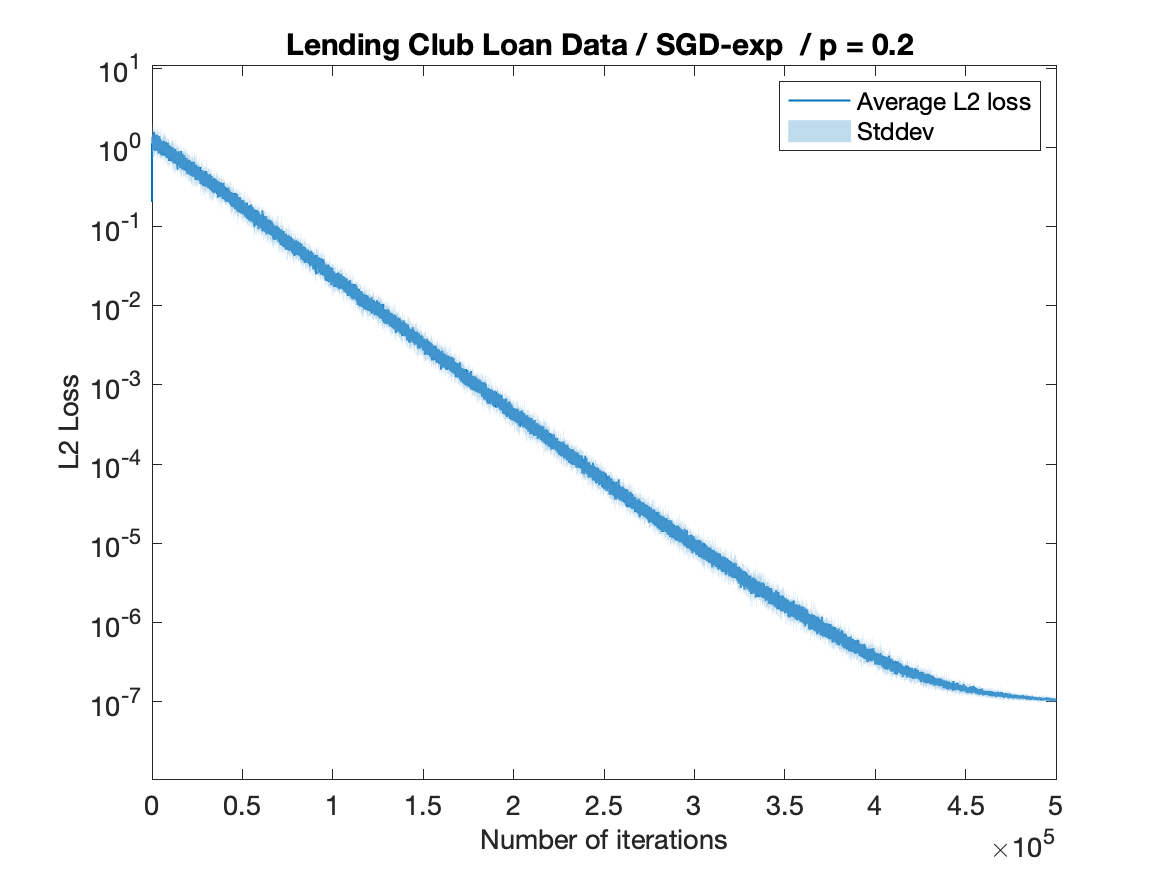}
      \end{minipage}
      
           \caption[] {Linear regression with SGD-exp on the Red Wine Quality dataset with the corruption probability $p = 0.2$ and $\lambda = 1.006$ (left) and on the Lending Club Loan dataset with the corruption probability $p = 0.2$ and $\lambda = 1.00002$ (right). The plots are averaged over $10$ runs.
}
        
    \label{fig:SGD-exp_real_dataset}
    \end{figure}
    \begin{figure} 
      \centering
      \includegraphics[width=0.45 \textwidth]{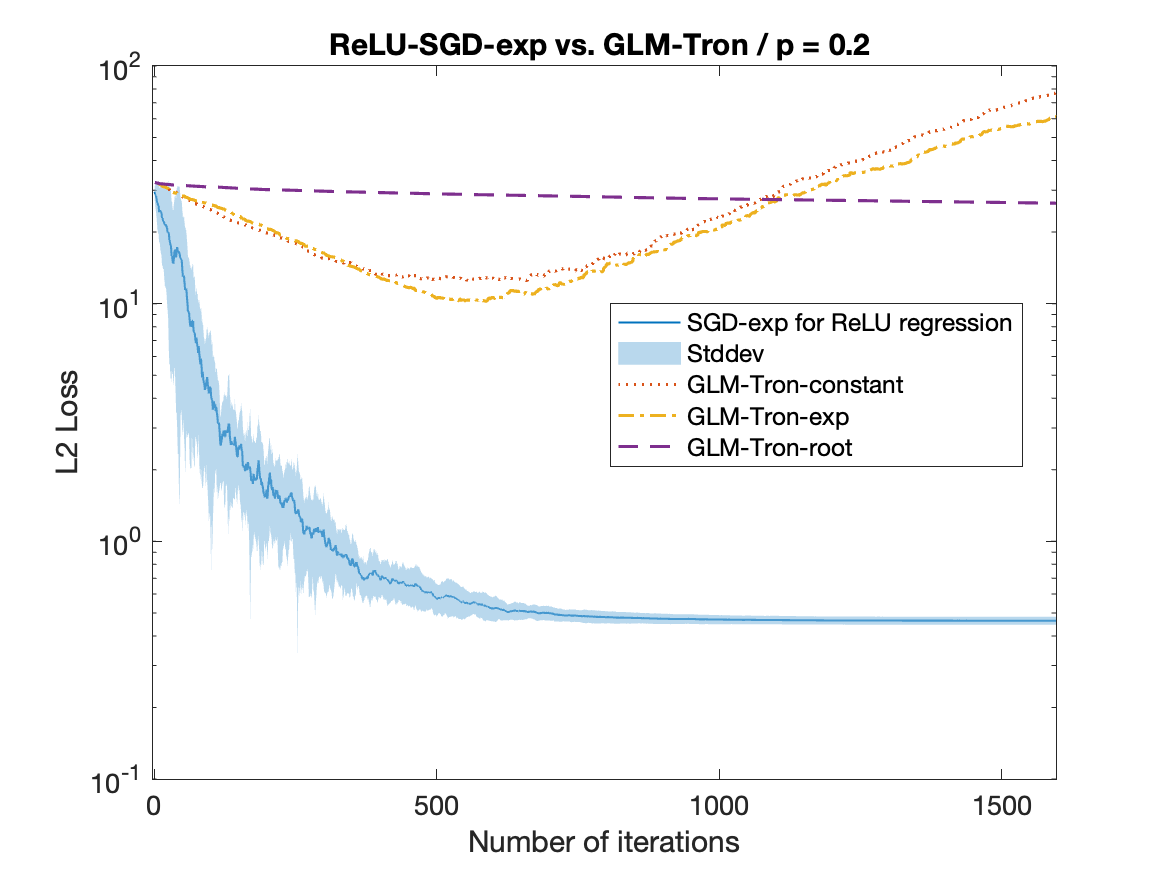}

           \caption[] {ReLU regression for the Red Wine Quality dataset with large Gaussian corruption with $p = 0.2$: SGD-exp with $\lambda = 1.006$ obtains low loss, whereas stochastic GLM-Tron with various step sizes suffers from noise. The number of samples in the data set $m = 1599$. Here, GLM-Tron-exp employs exponentially decaying step size $1.00003^{-k}/m$, GLM-Tron-const has step size $1/m$ and GLM-Tron-root has step size $k^{-1/2}/m$. The plots are averages over $10$ runs.     

  }\label{fig:Comparison_real_dataset_relu}
    \end{figure}

\subsubsection{Lending Club Loan Data}
Lending Club Loan dataset comprises $10000$ samples with loan data from $2007$ to $2015$ issued by Lending Club, a lending company in the US. There are several features in the data such as interest rates, loan amounts, balances, and so on. We use the first $34$ features in the dataset and the response variable is \verb|paid_total|. We normalize the feature/response vectors in the dataset, which is the typical preprocessing step for many machine learning algorithms. 

As before, the data points are randomly drawn from the dataset at each iteration of SGD-exp.
Under the random corruption model with probability $p=0.2$ by adding Gaussian noise drawn from $[-3,3]$ to the response variable and setting $\lambda = 1.00002$, we record the $\ell_2$ loss of SGD-exp in Figure \ref{fig:SGD-exp_real_dataset}. 
The loss value after $50,000$ iterations of SGD-exp is \verb|1.0116e-06|, which is the optimal value by running the linear regression on the uncorrupted dataset. On the other hand, the $\ell_2$ loss of linear regression on the corrupted dataset that is obtained by a direct solver is \verb|0.0017|. This demonstrates the effectiveness of our method, SGD-exp, in handling random corruption.

\section{Conclusion}\label{sec:conclusion}
In this paper, we have introduced Stochastic Gradient Descent with Exponential Decay (SGD-exp) for linear and ReLU regression in streaming settings under the presence of semi-random adversarial corruptions. Through theoretical analysis and numerical experiments, we have established that SGD-exp offers near-linear convergence rates for corruption probabilities less than $0.5$ for the Massart model and $1$ for the symmetric oblivious corruption model, optimal for both cases. Future research avenues include exploring SGD-exp's application to other robust optimization problems, further refining the convergence analysis under different noise models, and extending the framework to accommodate additional forms of non-linearity and constraints.

    \section*{Acknowledgements}
DN is supported by NSF DMS 2408912. ER is supported by the NSF grant DMS-2309685.

	\bibliographystyle{plain}
	\bibliography{arXiv_version_final}

\end{document}